%% file: main.tex
\documentclass[final]{neus2025}

\title[TM PINNs for ODEs]{Taylor-Model Physics-Informed Neural Networks (PINNs) for Ordinary Differential Equations}

\usepackage{times}
\usepackage{todonotes}
\usepackage{comment}
\usepackage{amsmath,amssymb, amsfonts}
\usepackage{upgreek}
\usepackage[skip=18pt]{caption}  
\usepackage{breqn, paralist}
\usepackage[title]{appendix}
\usepackage{algorithm, algorithmic}

\coltauthor{%
 \Name{Chandra Kanth Nagesh}\and\Name{Sriram Sankaranarayanan} \Email{first.lastname@colorado.edu}\\
 \addr University of Colorado Boulder
 \AND
 \Name{Ramneet Kaur}, \Name{Tuhin Sahai}\and\Name{Susmit Jha} \Email{first.lastname@sri.com}\\
 \addr SRI International
}

\input{commands}

\begin{document}

\maketitle

\vspace*{-0.5em}
\input{sections/abstract}

\input{sections/introduction}

\input{sections/preliminaries}

\input{sections/higher-order-pinns}

\input{sections/results}

\input{sections/conclusions}

\input{sections/ack}

\bibliography{references}

\newpage
\input{sections/appendix}

\end{document}

%% file: commands.tex
\usepackage[normalem]{ulem} 

\newcommand\vx{\vec{x}}

\newcommand\vt{\vec{t}}
\newcommand\vb{\vec{b}}
\newcommand\vtheta{\vec{\theta}}
\newcommand\reals{\mathbb{R}}
\newcommand\Lie{\mathcal{L}}
\newcommand\mrx{\mathrm{x}}
\newcommand\mry{\mathrm{y}}

%% file: sections/abstract.tex
\begin{abstract}
We study the problem of learning neural network models for Ordinary Differential Equations (ODEs) with parametric uncertainties. Such neural network models capture the solution to the ODE over a given set of parameters, initial conditions, and range of times. Physics-Informed Neural Networks (PINNs) have emerged as a promising approach for learning such models that combine data-driven deep learning with symbolic physics models in a principled manner.  However, the accuracy of PINNs degrade when they are used to solve an entire family of initial value problems characterized by varying parameters and initial conditions. 

In this paper, we combine symbolic differentiation and Taylor series methods to propose a class of higher-order models for capturing the solutions to ODEs. These models combine neural networks and symbolic terms: they use higher order Lie derivatives and a Taylor series expansion obtained symbolically, with the remainder term modeled as a neural network. The key insight is that the remainder term can itself be modeled as a solution to a first-order ODE. We show how the use of these higher order PINNs can improve accuracy using interesting, but challenging ODE benchmarks. We also show that the resulting model can be quite useful for situations such as controlling uncertain physical systems modeled as ODEs.
\end{abstract}


%% file: sections/introduction.tex
\section{Introduction}

Finding closed-form analytic solutions to systems of Ordinary Differential Equations (ODEs) is challenging for all but the simplest class of systems. The problem is even more challenging for ODEs with parameters that can take on a set of possible values, unknown initial conditions and external inputs. Physics-Informed Neural Networks (PINNs) have emerged as a solution to finding approximate closed-forms modeled as neural networks~\cite{raissi2018multistepneuralnetworksdatadriven}. They have been studied for solving PDEs, especially non-linear PDEs that are hard to solve numerically. The problem of PINNs for ODEs have received considerably less attention since numerical solvers are quite successful in finding solutions for many common ODEs~\cite{Hairer+Others/2000/Solving}. However, the use of numerical solvers is distinctly problematic for applications that involve solving optimization problems involving ODEs with \textit{unknown parameters and inputs}. These arise in machine learning, where one wishes to learn parameters from data or optimal control, wherein we seek control inputs that optimize a function across the trajectories of the system. For such applications, it is desirable to have a ``surrogate model'' that can capture the solution of the ODEs with high enough accuracy over a range of parameters, initial conditions and times. If each evaluation of the surrogate can be performed more efficiently than using a numerical solver, the overall optimization can be faster. 

We investigate the use of PINNs to build surrogate models that capture the solution to an initial value problem (IVP) given by a system of ODEs $\dot{\vx} = f(\vx, \vtheta, t)$ for parameters $\vtheta \in \Theta$, initial conditions $\vx(0) \in \Omega$ and $t \in [0, T]$. In other words, our surrogate model $\varphi(\vx_0, \vtheta, t)$ maps the inputs to a solution $\vx(t; \vx_0, \vtheta)$ of the ODE at time $t$.  The standard PINN approach of~\cite{raissi_physics-informed_2019} a) uses a neural network to represent $\varphi$ and b) a combination of two loss functions given by the initial condition loss $\| \varphi(\vx_0, \vtheta, 0) - \vx_0 \|$ and the gradient loss $\| \dot{\varphi} - f(\varphi, \vtheta, t)\|$ averaged at various randomly chosen ``collocation points''. 

In this paper, we first point out the inadequacy of the PINN approach to this problem by demonstrating its failure to approximate the solution when the sets $\Theta, \Omega$ and $[0, T]$ are large. We show that higher-order loss functions fail to address the issue. Therefore, we resort to a symbolic approach that uses successive Lie derivatives to compute the terms of a Taylor series expansion of the solution. We show that the remainder when carefully modeled can be written down as the solution to a derived ODE. Solving this derived ODE for the higher-order remainder using the ``classic'' PINN approach yields a solution that combines the best aspects of symbolic differentiation with neural network learning. We show that our error grows as $O(t^{m+1} e^{Kt})$ for an approach that uses derivatives up to order $m \geq 1$, whereas, for PINNs, the error grows as $O(t e^{Kt})$. As a result, our approach provides high levels of accuracy at the initial times. We compare our approach, which we call ``Taylor-Model PINNs'' with PINNs, and the related approach of ``Higher-Order PINNs'' that extends the original PINN loss function with higher order derivative-based loss functions. We show that Taylor-Model PINNs provide higher accuracy. While our approach increases the complexity of the training process, the use of efficient symbolic differentiation tools offsets this process. 

\subsection{Related Work}

Machine learning approaches have found applications in diverse domains ranging from celestial object classification \cite{angeloudi_multimodal_nodate}, climate forecasting \cite{iglesias-suarez_causally-informed_2024}, and tumor identification \cite{li_medical_2023}. In such scenarios, it has been observed that using background knowledge in the form of mathematical models in the learning process can considerably speed up convergence and improve solution quality. 


Physics-Informed Neural Networks (PINNs) \cite{raissi_physics-informed_2019} represent a seminal contribution in this space. They have been effective in solving systems which involve partial differential equations (PDEs), where data is sparsely available. By utilizing differentiable loss functions, a neural network is trained on the PDE residual and boundary condition loss to learn a solution map to the PDE. These physics-inspired loss terms act as a regularizer against learning solution maps that do not involve the underlying dynamics of the system, thereby conforming well to how the system evolves over time. This promising approach has led to widespread use of the methodology in various applications \cite{shukla_physics-informed_2020}, \cite{wang_deep_2021}, \cite{yin_non-invasive_2021}. 

Despite their contributions, PINN methodologies can often  fail to learn physical dynamics in many cases \cite{krishnapriyan_characterizing_2021}, \cite{steger2022how}. To tackle such issues, there have been efforts into  PINNs with additional loss functions \cite{son_enhanced_2023}, \cite{wang_is_nodate}. However, their efficacy diminishes under two conditions (a) conflicting gradient updates between the two loss functions, leading to suboptimal gradient descent \cite{hwang_dual_2025} and (b) when applied to parametric PDE families, particularly those requiring simultaneous resolution across a range of initial conditions and parameters \cite{xiang_physics-informed_2025}. 



The Taylor series expansion represents a fundamental concept in mathematical analysis, providing a powerful framework for approximating  functions through polynomials or power-series~\cite{Apostol/1967/Calculus}. 
Our approach of using Taylor series expansions has been heavily influenced by the work of Makino and Berz, who have applied so-called ``Taylor-model calculus'' to represent a set of complex and unknown functions by a finite Taylor series expansion with an interval remainder~\cite{Berz+Makino/1998/Verified,Makino+Berz/2009/Taylor}. This has led to popular approaches in the area of formal methods for proving properties of Cyber-Physical Systems~\cite{Chen+Sankaranarayanan/2022/Reachability,Althoff+Frehse+Girard/2021/Set,Althoff/2015/CORA,Kong+/2015/dReach}. Here, we adapt Taylor models to represent solutions but let the remainder be represented by a neural network rather than an interval. 
Parts of this problem have been investigated before, where trained neural networks are approximated using Taylor polynomials to enable integration of physical constraints into dynamical systems \cite{zhu_jing_leve_ferrari_2022}, \cite{balduzzi_mcwilliams_butler-yeoman_2016}. Furthermore, researchers have looked into the idea of Taylor layers for Transformer architectures \cite{zwerschke_weyrauch_götz_debus_2024}, which are higher order polynomial approximation replacements of standard linear or attention layers. The contributions of this paper can be summarized as follows:
\begin{compactenum}
    \item We show how a symbolic Taylor series expansion of an \emph{a priori} fixed order and the remainder term modeled by a neural network can be used to capture solutions of ODEs accurately.
    \item We present an evaluation of our approach based on how PINN errors grows over time.
    \item We develop a novel neural network training method and show that this new approach converges to tight solution maps compared to traditional PINNs, on seven ODE systems with varying dimensionality and parameter space. 
    \item The trends predicted by our analysis are empirically demonstrated on a set of examples. 
\end{compactenum}

%% file: sections/preliminaries.tex
\section{Preliminaries}

We will present some preliminary facts about ODEs, their solutions, and Taylor series expansions. 

\begin{definition}[Ordinary Differential Equations ]
A system of (coupled) Ordinary Differential Equations (ODE) over state variables $\vec{x} = (x_1, \ldots, x_n)$ and parameters $ \vec{\theta} = (\theta_1, \ldots, \theta_m)$ is of the form  $ \frac{d \vec{x}}{dt}  =  f(\vec{x}, \vec{\theta}, t)$, 
wherein $t$ represents the time variable, and $f: \mathbb{R}^n \times \mathbb{R}^m \times \mathbb{R} \rightarrow \mathbb{R}^n$ is a \emph{vector-field} that maps states, parameters and time to the value of the derivative.
\end{definition}
For an initial condition $\vx(0) = \vx_0$ and fixed values of the parameters $\vtheta$, 
the solution of the ODE is a differentiable map $\psi(\vx_0, \vec{\theta}, t)$ such that for all time $t$,  
$ \frac{d \psi(\vx_0, \vtheta, t)}{dt}  = f(\psi(x_0, \vtheta, t), \vtheta, t)$.
We assume that the function $f$ defining the RHS of the ODE is \emph{Lipschitz} continuous, thus ensuring the existence and uniqueness of solutions.

\begin{example}[Duffing Oscillator]\label{Ex:duffing-oscillator-model}
Consider an ODE that models the dynamics of a Duffing oscillator with $\vx = (x, y)$, $\vtheta=(\delta)$ and dynamics given by $ \frac{dx}{dt} \ =\ y,\ \frac{dy}{dt} \ = \ x - x^3 - \delta y $. 
\end{example}

Given an ODE model, we seek to represent the solution $\varphi$ as a function of all initial conditions $\vx_0$, parameters $\theta$ and time $t$. However, this sort of \emph{analytical solution} is only available for a restricted class of ODEs. In practice, we have to settle for an \emph{approximate  solution} available either through a numerical ODE solver (for instance, using the Runge-Kutta algorithm) or an approximate analytic solution $\varphi(\vx_0, \vtheta, t)$ that provides a solution close to the real solution for a range of initial conditions $\vx_0 \in \Omega$, $\vtheta \in \Theta$ and $t \in [0, T]$ for given sets $\Omega, \Theta$ and time horizon $T > 0$.  The approximate solution map has two potential advantages: (a) it can be computationally less expensive than numerical solvers; and (b) it can be used to estimate derivatives such as $\frac{d \varphi(t)}{d \vtheta}$ and $\frac{d \varphi(t)}{d \vx_0}$ efficiently. Computing such derivatives is useful for learning parameters from data, and  is hard to do using numerical solvers. The main problem statement for this paper is as follows:

\begin{definition}[Learning Solution Map]
Given the description of an ODE $\frac{d\vx}{dt} = f(\vx, \vtheta, t)$, a set of initial conditions $\vx_0 \in \Omega$, a set of parameters $\vtheta \in \Theta$ and a time horizon $t \in [0, T]$ for $T > 0$, we wish to learn a solution map $\varphi: \reals^n \times \reals^m \times \reals \rightarrow \reals^n$ such that $\varphi(\vx_0, \theta, t)$ is as close as possible to the precise solution map $\psi$.
\end{definition}

The physics-informed approach learns $\varphi$ by fixing a finite sample set of ``collocation points'' $S \subseteq \Omega \times \Theta \times [0, T]$, wherein $|S| = N$, and minimizing two different loss functions simultaneously:
\[ \begin{array}{rl}
L_{i} & = \frac{1}{N} \sum_{(\vx_0, \vtheta, t) \in S} \| \varphi(\vx_0, \vtheta, 0) - \vx_0 \|\,, \\ 
L_{g} & = \frac{1}{N} \sum_{(\vx_0, \vtheta, t) \in S} \| 
\frac{\partial}{\partial t} \varphi(\vx_0, \varphi, t)  - f(\varphi(\vx_0, \vtheta, t) , \vtheta, t)  
\|\,\\ 
\end{array}\]
In practice, the approaches minimizes a linear combination $L= \omega_i L_i + \omega_g L_g$ for  user-specified constants $\omega_i, \omega_g$. The process of learning $\varphi$ proceeds by fixing a neural network architecture with unknown weights $W$ and using stochastic gradient descent (since $N$ is typically very large) technique to find a local minimizer $W^*$ for the overall loss $L$.
\begin{figure}[ht!]
    \centering
    \includegraphics[trim=0 15 10 10,clip,width=\linewidth]{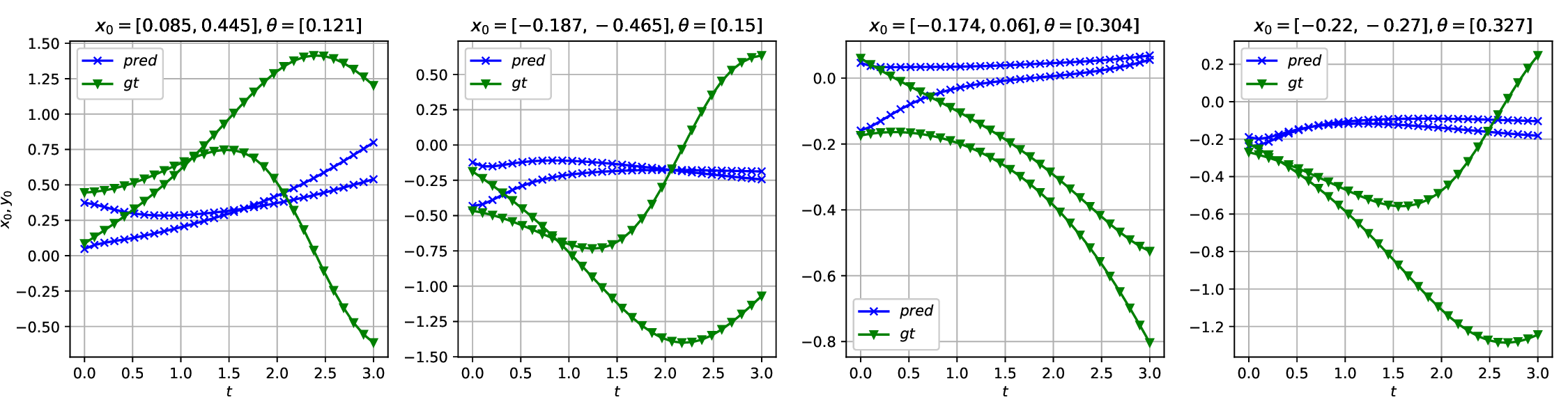}
    \caption{Numerical simulations (taken as ground-truth) shown in green compared against the PINN prediction shown in blue for the Duffing Oscillator system provided in Ex.~\ref{Ex:duffing-oscillator-model}. The initial conditions and parameters are randomly sampled from $\Omega \in [-0.5, 0.5]$ and $\theta \in [0.1, 0.5]$ respectively.}
    \label{fig:do-pinn}
\end{figure}
\begin{example}
Consider the Duffing oscillator model from Ex.~\ref{Ex:duffing-oscillator-model}.
Fig.~\ref{fig:do-pinn} compares the ``ground truth'' trajectories obtained through a numerical simulation against the predictions of the PINN model given by a neural network with 1 layers and 64 neurons per layer. Note that the trajectories diverge rapidly from predictions. Also, note that since the loss $L_i$ is not zero, it causes discrepancies even in the initial predictions.
\end{example}

Let us assume that we are able to learn a neural network model $\varphi_N(\vx_0, \vtheta, t)$ for $\vx_0 \in \Omega$, $\vtheta \in \Theta $ and $t \in [0, T]$ such that $\varphi_N$ is a differentiable function of $t$, and the following inequalities hold:
\begin{align*}
    \max_{\vx_0 \in \Omega, \vtheta \in \Theta} \| \varphi_N(\vx_0, \vtheta, 0) - \vx_0 \| \leq L_{i, \max}\; \text{and}\;
    \max_{\vx_0 \in \Omega, \vtheta \in \Theta, t \in [0, T]} \| \dot{\varphi}_N(\vx_0, \vtheta, t) - f(\varphi_N, \vtheta, t) \| \leq L_{g, \max}
\end{align*}
for some constants $L_{i,\max}, L_{g, \max} > 0$.
Let $\psi(\vx_0, \vtheta, t)$ represent the analytical solution of the ODE.
\begin{theorem}\label{thm:useful-theorem-1}
There exists a constant $K > 0$ such that, for all $\vx_0 \in \Omega, \vtheta \in \Theta$ and $t \in [0, T]$, we have  $\| \varphi(\vx_0, \vtheta, t) - \psi(\vx_0, \vtheta, t) \| \leq (L_{i, \max} + L_{g, \max} t) e^{Kt}$
\end{theorem}
\begin{proof}
Note that $ \psi(\vx_0, \vtheta, t) = \vx_0 + \int_{0}^t f(\psi(\vx_0, \vtheta, s), \vtheta, s) ds$. Likewise, assuming differentiability of $\varphi_N$, we have 
$ \varphi_N(\vx_0, \vtheta, t)= \varphi_N(\vx_0, \vtheta, 0) + \int_{0}^t \dot{\varphi}_N(\vx_0, \vtheta, s) ds$.  For simplicity, we write $\varphi_N(t) := \varphi_N(\vx_0, \vtheta, t)$ and $\psi(t):= \psi(\vx_0, \vtheta, t)$. We have,
\begin{align*}
\| \varphi_N(t) - \psi(t) \| & \leq \| \varphi_N(0) - \vx_0 \| + \left\| \int_{0}^t  \dot\varphi_N(s) - f(\psi(s), \vtheta, s)  ds \right\|  \\
& \leq  K_{i, \max} + \int_{0}^t \| \dot\varphi_N(\vx_0, \vtheta, s) - f(\psi, \vtheta, s) \| ds \\ 
& \leq K_{i,\max} + \int_{0}^t \| \dot\varphi_N(s) - f(\varphi_N(s), \vtheta, s) \| ds  + 
\int_{0}^t \| f(\varphi_N, \vtheta, s) - f(\psi, \vtheta, s) \| ds \\
& \leq K_{i, \max} + K_{g, \max} t + L_f \int_{0}^t \| \varphi_N(s) - \psi(s) \|  ds
\end{align*}
wherein $L_f$ is the Lipschitz constant of $f(\vx, \vtheta, t)$ over, $\vx \in X$ obtained as $\varphi(\Omega, \Theta, [0, T]) \cup \psi(\Omega, \Theta, [0, T])$. This set will be compact if $\Omega, \Theta$ are compact and $T$ is finite. Applying Gr\"onwall's inequality~\cite{Bellman/1943/Stability}, we conclude, 
$\| \varphi_N - \psi \| \leq (K_{i, \max} + K_{g, \max}t) e^{L_f t}$. 
\end{proof}

%% file: sections/higher-order-pinns.tex
\section{Higher-Order PINNs}

In this section, we will tackle the problem of using symbolic differentiation computations based on Taylor series methods. Given ODE $\dot{\vx} = f(\vx, \vtheta, t)$, recall that the Lie derivative of a function $g(\vx, t)$ is given by 
$ \Lie_f (g) = \left(\nabla_{\vx} g \right) \cdot f + \frac{\partial}{\partial t } g $. 
We define the successive Lie derivatives: $\Lie^{(0)}(\vx) = f_0(\vx) = \vx$, and  $\Lie^{(i+1)}(\vx) = \Lie(f_i(\vx, \vtheta, t)) := f_{i+1}(\vx, \vtheta, t) $.  

The main idea behind higher order PINNs is to consider loss functions beyond the first order loss function. For instance, the second and third order losses are defined as:
\[ \begin{array}{rl}
&L_2 = \frac{1}{N} \sum_{(\vx_0, \vtheta, t) \in S} \left\| \frac{d^2}{dt^2} \varphi(\vx_0, \vtheta, t)  - f_2( \varphi, \vtheta, t) \right\|\, ~\text{and}~ \\ 
&L_3 = \frac{1}{N} \sum_{(\vx_0, \vtheta, t) \in S} \left\| \frac{d^3}{dt^3} \varphi(\vx_0, \vtheta, t)  - f_3( \varphi, \vtheta, t) \right\|\,\\ 
\end{array} \]

The overall loss is obtained by combining the initial condition loss $L_i$, the PINN gradient loss $L_g$, and the higher order losses: 
$ L = \alpha_0 L_i + \alpha_1 L_g + \alpha_2 L_2 + \alpha_3 L_3 + \cdots  + \alpha_m L_m $. 

However, such a scheme has two main disadvantages: (a) it requires us to take higher order derivatives of a large and complex neural network model; and (b) it introduces multiple loss functions, all of which need to be minimized by selecting an appropriate linear combination of loss functions. We propose, instead, a simpler scheme based on Taylor series that has the advantage of (a) using a single loss function, (b) not requiring Hessians or higher-order gradients of neural networks and (c) tries to match the flow up to some order $m > 0$. 

\subsection{Higher-Order PINNs based on Taylor series}
We will assume through the rest of this paper that the RHS function $f$ is at least $m+2$ times differentiable for some $m > 0$, we have:
\[ \vx(t) = \psi(\vx_0, \vtheta, t) = \vx_0 + t f_1(\vx_0, \vtheta, 0 ) + \frac{t^2}{2!} f_2(\vx_0, \vtheta, 0) + \cdots + \frac{t^{m}}{m!} f_{m}(\vx_0, \vtheta, 0) + T_m(\vx_0, \vtheta, t) \, \]
wherein $T_m$ denotes the remainder term of order $m+1$. 

\begin{theorem}
The function $T_m$ satisfies the following properties:
\begin{enumerate}
\item $ T_m(\vx_0, \vtheta, t) = \frac{1}{m!} \int_{0}^t (t-s)^m f_{m+1}(\vx(s), \vtheta, s) ds $
\item $T_m(\vx_0, \vtheta, 0) =  \dot{T}_m(\vx_0, \vtheta, 0) = \cdots = T_m^{(m)}(\vx_0, \vtheta, 0) = 0$ \,
\end{enumerate}
\end{theorem}
\begin{proof}
The proof of the first statement is available from~\cite[Theorem 7.6]{Apostol/1967/Calculus}. The second statement follows from repeatedly differentiating the RHS of the first equality using Leibnitz's rule for differentiation under the integral sign.
\end{proof}
Rather than use loss functions to enforce that $T_m^{(j)}(\vx_0, \vtheta, 0)= 0$ for $j \leq m$, we can write 
\[ T_m(\vx_0, \vtheta, t)  = \frac{t^{m+1}}{(m+1)!} R_m(\vx_0, \vtheta, t) \,, \text{wherein}\ R_m = (m+1) \int_{0}^t \frac{1}{t} \left( 1 - \frac{s}{t} \right)^m f_{m+1}(\vx(s), \vtheta, s) ds  \,\]
$R_m$ can be re-written using the change of variables $\alpha = \frac{s}{t}$ as 
\begin{equation}\label{eq:rm-quadrature}
R_{m} = (m+1) \int_{0}^1 (1 - \alpha)^m f_{m+1}(\vx(\alpha t), \vtheta, \alpha t) d \alpha \,
\end{equation}

The goal is to use a neural network model for $R_m$ while inferring its parameters through a loss function.  We propose two approaches: 
(a) an indirect approach based on the PINN loss function and (b) a direct approach that uses quadrature to approximate the integral in Eq.~\eqref{eq:rm-quadrature}.

\paragraph{PINN-based loss function:} \label{par:tmpinnloss}
We will use a PINN to model $R_m(\vx_0, \vtheta, t)$.  For convenience, we will assume that $\vx_0, \theta$ are fixed and denote $\tau_m(t) = T_m(\vx_0, \vtheta, t)$, $r_m(t) = R_m(\vx_0, \vtheta, t)$ and $f_i(0)$ denote $f_i(\vx_0, \vtheta, 0)$. Note that $\tau_m(t) = \frac{t^{m+1}}{(m+1)!} r_m(t)$. Let 
\begin{equation} \label{eq:tm-pinn-model}
    \varphi_r(\vx_0, \vtheta, t) = f_0(0) + t f_1(0) + \cdots + \frac{t^m}{m!} f_m(0) + \tau_m(t)
\end{equation}
\begin{theorem}
The remainder $\tau_m$ is a solution to the ODE with Lipschitz continuous RHS:
\[ \dot{\tau}_m(t) = f_1(\varphi_r, \vtheta, t)  - \left( f_1(0) + t f_2(0) + \cdots + \frac{t^{m-1}}{(m-1)!} f_m(0) \right) \,.  \]
Furthermore, it has the form $\tau_m(t) = \frac{t^{m+1}}{(m+1)!} r_m(t)$ for a continuous and differentiable function $r_m(t)$  with $r_m(0) =f_{m+1}(\vx_0, \vtheta, 0)$
\end{theorem}
\noindent
Now, we can use PINNs to learn the remainder $R_m$ as a function of time using the loss functions:
\begin{compactenum}
\item $ L_{r,g} = \frac{1}{N} \sum_{(\vx_0, \vtheta, t) \in S} \left\| f(\varphi_r, \theta, t) - \left( f_1 + t f_2 + \cdots + \frac{t^{m}}{(m)!} R_m + \frac{t^{m+1}}{(m+1)!} \dot{R}_m \right)  \right\| $,  
and 
\item $L_{r,i} = \frac{1}{N} \sum_{(\vx_0, \vtheta, t) \in S} \left\| f_{m+1}(\vx_0, \vtheta, 0) - R_m(\vx_0, \vtheta, 0) \right\|$
\end{compactenum}
\begin{example}
Consider the Duffing Oscillator case from Ex.~\ref{Ex:duffing-oscillator-model} 
with state variables  $x, y$ and parameter $\delta$. The overall system $\varphi (\vx_0, \delta, t)$ based on Taylor series expansion of order $m=4$ is as follows:
\begin{align*}
    \begin{bmatrix}
        x_0 + ty_0 + \frac{t^2}{2!}(x_0 -\delta y_0 - x_0^3) + \frac{t^3}{3!}f_2(\vx_0, \delta, t) + \frac{t^4}{4!}R_m(\vx_0, \delta, t) \\
        y_0 + t(x -\delta y - x^3) + \frac{t^2}{2!}f_2(\vx_0, \delta, t) + \frac{t^3}{3!}f_3(\vx_0, \delta, t) + \frac{t^4}{4!}R_m(\vx_0, \delta, t)
    \end{bmatrix}
\end{align*}
where, $f_2(\vx_0, \delta, t) = (-\delta(x_0 -\delta y_0 - x_0^3) + y_0(1 - 3x_0^2));~ f_3(\vx_0, \delta, t) = (y_0(-\delta(1 - 3x_0^2) - 6x_0y_0) + (\delta^2 - 3x_0^2 + 1)(-\delta y_0 - x_0^3 + x_0))$ are the second and third derivatives of the system and $R_m$ is the neural network model that is trained to learn the remainder term of the expansion.

Now, in order for the solution of the system $\varphi$ to match the original Duffing Oscillator, we can write  $\dot{\varphi} (\vx_0, \delta, t)$ as:
\begin{align*}
    \begin{bmatrix}
        y_0 + t(x_0 -\delta y_0 - x_0^3) + \frac{t^2}{2!}f_2(\vx_0, \delta, t) + \frac{t^3}{3!}R_m(\vx_0, \delta, t) + \frac{t^4}{4!}\dot{R}_m(\vx_0, \delta, t) \\
        (x_0 -\delta y_0 - x_0^3) + tf_2(\vx_0, \delta, t)  + \frac{t^2}{2!}f_3(\vx_0, \delta, t) + \frac{t^3}{3!}R_m(\vx_0, \delta, t) + \frac{t^4}{4!}\dot{R}_m(\vx_0, \delta, t)
    \end{bmatrix}
\end{align*}
where, $\dot{R}_m$ is the first-order time derivative of the neural network model. The overall loss function for the learning procedure, $L = L_{r,g} + L_{r, i}$ can now be calculated with the both sides of the equation computed as above. The implementation of the training algorithm is provided in Appendix~\ref{appendix:A3}.
\end{example}

\paragraph{Loss Function Through Numerical Quadrature:} Rather than differentiate $R_m$, we can use a numerical approach to directly encode the remainder formula in Eq.~\eqref{eq:rm-quadrature}. Let us subdivide the interval $\alpha \in [0, 1]$ into $K+1$ quadrature points, wherein $\alpha_k = \frac{k}{K}$ for $k \in \{ 0, \ldots, K\}$.  Using the trapezoidal rule, we obtain

\begin{equation}
    R_{m}(\vx_0, \vtheta, t) \approx \frac{m+1}{K} \left( \frac{1}{2} F(0) +  \sum_{k=1}^{K-1}  F\left(\frac{k}{K}\right) + \frac{1}{2} F(1)\right) \,
\end{equation}
wherein $F(\alpha) := (1-\alpha)^m f_{m+1}(\varphi(\vx_0, \vtheta, \alpha t), \vtheta, \alpha t)$. Note that $F(1) = 0$.
However, complexity of this approach depends on the choice of $K$: a small value of $K$ makes the quadrature highly erroneous whereas a larger value makes the approach quite expensive. Further investigation shows us that the quadrature method does not work well for smaller choices of $K$, however, is faster than our other approach. Detailed results and analysis of this behaviour is provided in Appendix~\ref{appendix:numericalquad}.

\paragraph{Analysis:} Let us assume that we have inferred a differentiable model $R_m$ which achieves a maximum possible loss  $ \max_{(\vx_0, \vtheta, t) \in \mathcal{S}} L_{r,g} = K_{rg,\max}$  and
$\max_{(\vx_0, \vtheta, t) \in \mathcal{S}} L_{r,i} = K_{r,i,\max}$ over the compact set of inputs
$\mathcal{S} = \Omega \times \Theta \times [0, T]$. 

\noindent
Let $\psi(\vx_0, \vtheta, t)$ represent the solution map for the ODE and $\varphi_r$ denote the model from Eq.~\eqref{eq:tm-pinn-model}. 
\begin{theorem}
 There exists a constant $K$ such that for all $(\vx_0, \vtheta, t) \in \mathcal{S}$,
\[ \| \varphi_r(\vx_0, \vtheta, t) - \psi(\vx_0, \vtheta, t) \| \leq \frac{t^{m+1}}{(m+1)!} ( K_{r, i} + K_{r,g} t) e^{Kt} \,\]
\end{theorem}
\begin{proof}
Applying Theorem~\ref{thm:useful-theorem-1} to the PINN learning problem for $R_m$, and letting $R_m^*$ be the exact remainder obtained from the ODE solution map $\psi$, we obtain that \[ \| R_m - R_m^* \| \leq ( K_{r,i} + K_{r,g}t) e^{Kt} \]
In turn, from Eq.~\eqref{eq:tm-pinn-model}, we obtain that $\|\varphi - \psi\| \leq  \frac{t^{m+1}}{(m+1)!} ( K_{r, i} + K_{r,g} t) e^{Kt}$.
\end{proof}

%% file: sections/results.tex
\section{Results}

In this section, we present results of running the three models against seven numerical ODE benchmark systems with varying dimensionality and number of parameters. The seven systems are as follows: a) \textit{Duffing Oscillator}, b) \textit{Damped Pendulum}, c) \textit{Lorenz Attractor}, d) \textit{Lotka-Volterra system}, e) \textit{Rikitake Model}, f) \textit{Susceptible-Infected-Recovered (SIR) Model} g) \textit{Susceptible-Exposed-Infected-Recovered (SEIR) model}. Further, we show that our method scales well on higher dimensional systems. The results and prediction performance are very similar to the seven benchmarks ODEs mentioned above and their analysis can be found in Appendix~\ref{appendix:A2_3}.

Detailed definitions for all the seven dynamical system along with the initial condition range $\Omega$ and parameter range $\Theta$ are provided in Appendix \ref{appendix:A1}. The PINN, Higher-Order PINN (HO-PINN), and Taylor-Model PINN (TM-PINN) are all represented using a 1 layer, 64 hidden unit shallow neural network. The input sizes vary according to the dynamical system, however, for TM-PINN we use separate neural networks to learn each dimension of the system. All the models were trained on Apple M3 Pro 14-Core GPU, running JAX 0.4.x with Metal 3 support. The seeds for all dataset creation and model initialization are provided in this \href{https://github.com/cuplv/TM-PINNs}{codebase}. 

\begin{figure}[ht!]
    \centering
    \includegraphics[trim=0 15 10 10,clip,width=\linewidth]{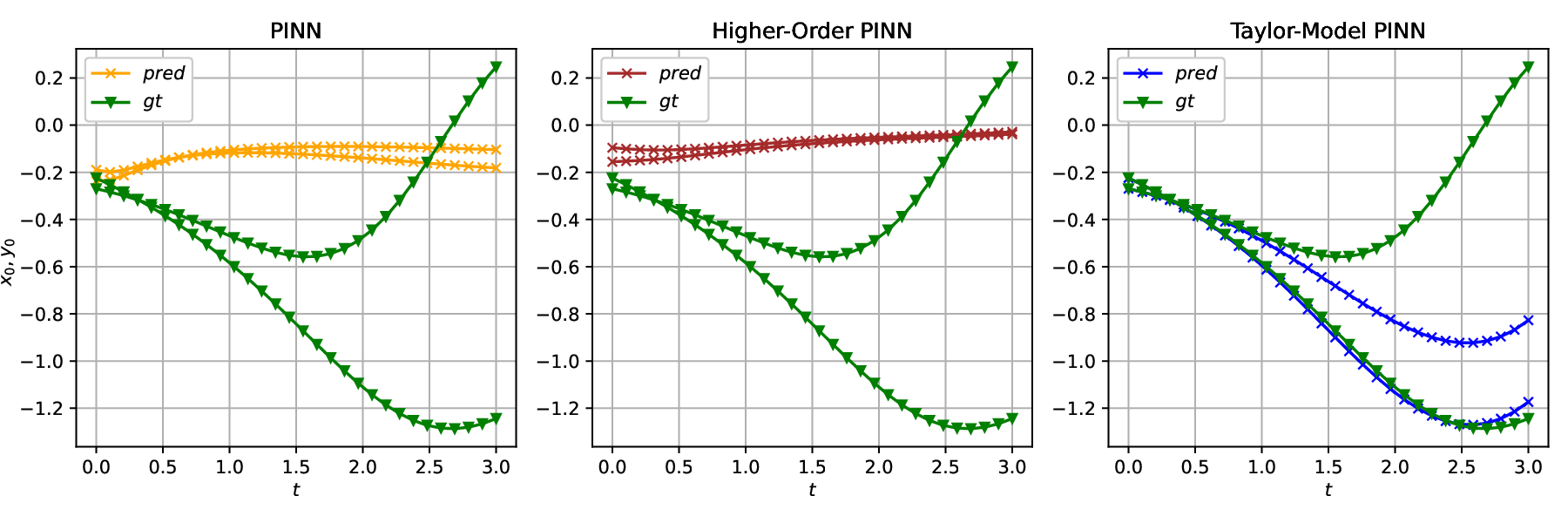}
    \caption{Prediction performance of the three models on \textbf{Duffing Oscillator} system. The initial condition are set to $\vx_0=[-0.224, -0.269]$ and $\vtheta_0=[0.327]$. \{``gt"=ground truth, ``pred"=prediction\}}
    \label{fig:do-model-predictions}
\end{figure}

\begin{figure}[ht!]
    \centering
    \includegraphics[trim=0 15 10 10,clip,width=\linewidth]{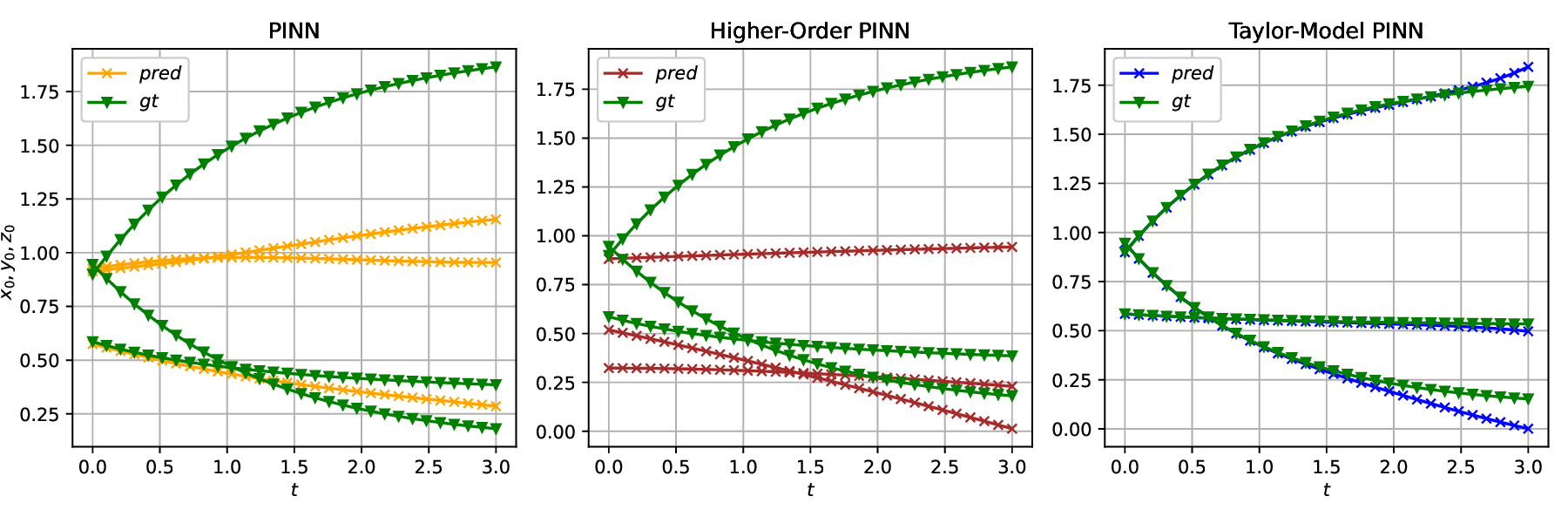}
    \caption{Prediction performance of the three models on \textbf{SIR} system. The initial conditions are set to $\vx_0=[0.586, 0.945, 0.899]$ and $\vtheta_0=[0.81, 0.1]$. \{``gt"=ground truth, ``pred"=prediction\}}
    \label{fig:sir-model-predictions}
\end{figure}

\begin{figure}[ht!]
    \centering
    \includegraphics[trim=0 15 10 10,clip,width=\linewidth]{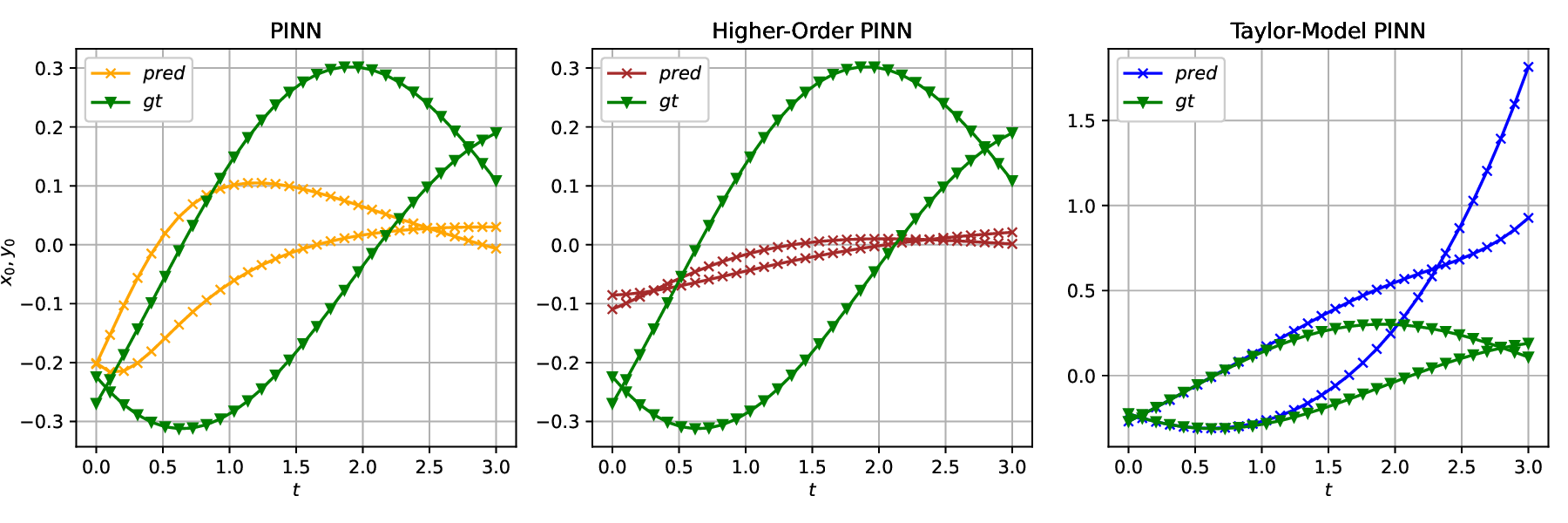}
    \caption{Prediction performance of the three models on \textbf{Damped Pendulum} system. The initial conditions and parameters are set to $\vx_0=[-0.224, -0.269]$ and $\vtheta_0=[0.288, 7.2]$. \{``gt"=ground truth, ``pred"=prediction\}}
    \label{fig:dp-model-predictions}
\end{figure}

\subsection{Exact collocation point match for shorter time periods}

We can first understand the performance of these models by looking at collocation points for shorter time windows. In Fig.~\ref{fig:do-model-predictions},~\ref{fig:sir-model-predictions},~\ref{fig:dp-model-predictions} we can observe the performance of standard PINNs does not really match the collocation points after the initial few sample points. This behavior can also be observed with HO-PINNs where the initial conditions also do not tend to match. However, with further qualitative results from Tab.~\ref{tab:results} we can observe that TM-PINNs seem to match well with the system output for larger time horizons, up to two seconds for all the benchmark systems.

\begin{table}[ht!]
    \centering
    \small
    \begin{tabular}{lcccccccc}
        \hline
        Method & DO (2,1) & DP (2,2) & LV (2,4) & R (3,2) & SIR (3,2) & LoA (3,3) & SEIR (4,6) & Time (sec) \\ \hline
        PINN & 0.130 & 0.109 & 0.035 & 0.038 & 0.08 & 0.019 & \textbf{0.069} & 1 \\
        HO-PINN & 0.141 & 0.210 & 0.070 & 0.042 & 0.329 & 0.029 & 246.2 & 1 \\
        TM-PINN & \textbf{0.003} & \textbf{0.012} & \textbf{0.004} & \textbf{0.016} & \textbf{0.002} & \textbf{0.008} & 0.784 & 1 \\ \hline
        PINN & 0.229 & \textbf{0.143} & 0.079 & 0.09 & 0.146 & \textbf{0.022} & \textbf{0.135} & 2 \\
        HO-PINN & 0.24 & 0.215 & 0.133 & 0.09 & 0.334 & 0.035 & 309.2 & 2 \\
        TM-PINN & \textbf{0.048} & 0.185 & \textbf{0.06} & \textbf{0.085} & \textbf{0.023} & 0.074 & 12.28 & 2 \\ \hline
        PINN & 0.312 & \textbf{0.16} & \textbf{0.169} & 0.247 & 0.188 & \textbf{0.025} & \textbf{0.202} & 3 \\
        HO-PINN & 0.329 & 0.21 & 0.251 & 0.248 & 0.354 & 0.038 & 376.4 & 3 \\
        TM-PINN & \textbf{0.170} & 0.677 & 0.453 & \textbf{0.206} & \textbf{0.079} & 0.220 & 55.17 & 3 \\ \hline
    \end{tabular}
    \caption{Results showing (MAE$\downarrow$) on Taylor-Model PINN (TM-PINN) compared against vanilla PINN and Higher-Order PINN (HO-PINN) on seven different dynamical system models across varying prediction time. \{DO=Duffing Oscillator, DP=Damped Pendulum, LV=Lotka-Volterra, R=Rikitake, SIR=Susceptible-Infected-Recovered, LoA=Lorenz Attractor, SEIR=S-Exposed-IR\}. The two numbers next to each model name in the header row  show the number of state variables and parameters, respectively.}
    \label{tab:results}
\end{table}

To understand why TM-PINNs exhibit poorer performance over longer time horizons Fig.~\ref{fig:error-prop}, we can examine the third graph in Fig.~\ref{fig:dp-model-predictions}. This figure shows that TM-PINNs accurately capture the evolving dynamics of the Damped Pendulum model up to approximately $1.4$ seconds. However, beyond this point, the residual term fails to converge effectively, causing predictions to exceed the expected output range of this dynamical system. This observation is further supported by the error plots and metrics presented in Appendix~\ref{appendix:A2}.  
In general, we observe that TM-PINNs achieve convergence significantly faster than the other methods. However, each epoch is significantly slower due to the more complicated loss functions that involve additional terms. We find that on average across all systems, TM-PINNs require approximately $8$ minutes to train for an average of $500$ epochs. In comparison, PINNs take 2 minutes to complete $10^5$ epochs, while HO-PINNs require $8$ minutes. 
\begin{figure}[ht!]
    \centering
    \includegraphics[trim=0 15 10 10,clip,width=0.9\linewidth]{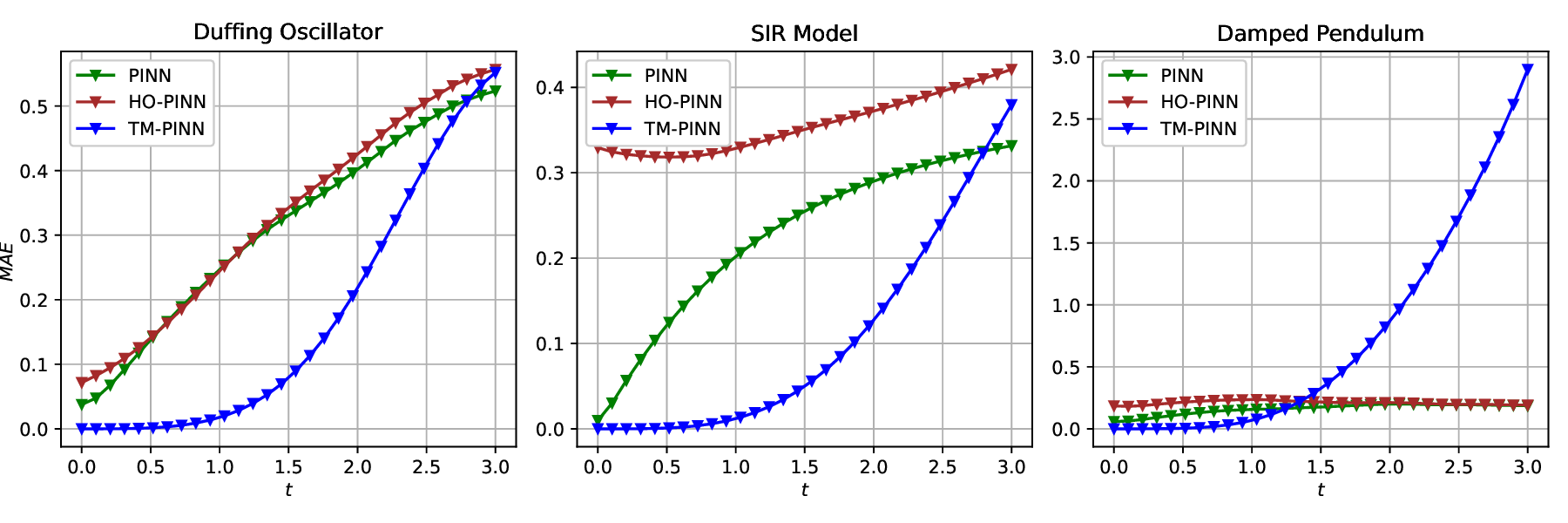}
    \caption{Avg. MAE plotted at various time points throughout the simulation of PINN, HO-PINN and TM-PNN (our approach) on three different ODEs presented in Fig.~\ref{fig:do-model-predictions},~\ref{fig:sir-model-predictions},~\ref{fig:dp-model-predictions}.}
    \label{fig:error-prop}
\end{figure}

%% file: sections/conclusions.tex
\vspace{-0.5cm}

\section{Conclusion}

We have presented an approach that uses Taylor models to cast the problem of learning solutions as one for learning the remainder term from the Taylor series expansion of the solution. A theoretical analysis of our approach yields error bounds that indicate that the approach can be quite effective over smaller time horizons, while PINNs have an error growth that makes them less error-prone over longer time horizons. The extension of our work to solving PDEs using Taylor series expansions and alternatives to characterizing the remainder term remains an important part of our future work. The challenge therein lies in carefully characterizing the boundary conditions and initial conditions for the higher-order terms in the Taylor series expansion of the PDE solution.  We are also interested in other types of series expansions that can approximate the ODE solutions, such as Fourier series expansions, expansions based on special functions, especially for ODEs/PDEs with oscillatory solutions~\cite{Agarwal+ORegan/2009/Ordinary}. Specialized techniques such as power-series expansions based on Koopman operators and  convergent power series expansions, especially for Lotka-Volterra-type systems, are also amenable to the approaches developed in this paper~\cite{Basor+Morrison/697/Analytic}.

%% file: sections/ack.tex
\acks{We would like to acknowledge the valuable discussions, feedback, and resources provided by our colleagues and external collaborators through out the process. This material is based upon work supported by the United States Air Force and DARPA under Contract No. FA8750-23-C-0519 and HR0011-24-9-0424, and the U.S. Army Research Laboratory under Cooperative Research Agreement W911NF-17-2-0196, the US National Science Foundation (NSF)  under awards \# CCF-2422136 and CPS-1836900, and  NCCIH grant \# R01AT012288. Any opinions, findings, and conclusions expressed are those of the author(s) and do not necessarily reflect the views of the United States Air Force, DARPA, the U.S. Army Research Laboratory, or the United States Government.}

%% file: sections/appendix.tex
\begin{appendices}
\appendix \section{Benchmark systems} \label{appendix:A1}

This section gives a detailed explanation of the seven benchmark ODEs that were considered for the experiments. We have mentioned the parametric and initial condition ranges that were considered for training the models. Further, we made sure to fix a unique seed value for generating all the datasets in the ranges mentioned below.
\subsection{Duffing Oscillator}\label{method:do}
A nonlinear differential equation that is used to represent the dynamics of a damped oscillator is called a Duffing equation. It can be represented in two-dimensional form as follows:
\begin{align*}
    &\dot{\mathrm{x}} = y ~;~ \dot{\mathrm{y}} = x - x^3 - \delta y ~;~~ (x, y) \in \Omega \\
    &\Omega \in [-0.5, 0.5] \times [-0.5, 0.5]
\end{align*}
where, $\delta \in [0.1, 0.5]$ is the damping factor. The initial conditions and parameters of the model are uniform-randomly sampled from the range. The second derivatives w.r.t. time $t$ and the Lie derivatives w.r.t the parameter of the model $\delta$ as follows:
\begin{align*}
    &\ddot{\mathrm{x}} = x - x^3 - \delta y ~;~ \ddot{\mathrm{y}} = (1 - 3x^2)y - \delta  (x - x^3 - \delta y) \\
    &\mathrm{x}_{\delta} = 0 ~;~ \mathrm{y}_{\delta} = - y
\end{align*}

\subsection{Damped Pendulum}\label{method:dp}
The equations of motion of a pendulum can be represented in two-dimensional form as follows:
\begin{align*}
    &\dot \theta = \omega ~;~ \dot \omega = -b \omega - \frac{g}{L} \sin \theta ~;~~ (\theta, \omega) \in \Omega \\
    &\Omega \in [-0.5, 0.5] \times [-0.5, 0.5]
\end{align*}
where, $L \in [1, 10]$ is the length of the string attached to the pendulum and $b \in [0.01, 0.5]$ is the air resistance, $\theta$ and $\omega$ are the state variables representing the angle and angular momentum of the pendulum. The initial conditions and parameters of the model are similarly sampled from the range. The second derivatives w.r.t. time $t$ and the Lie derivatives w.r.t the parameters $b, L$ as follows:
\begin{align*}
    &\ddot{\theta} = -\frac{g sin\theta}{L} - b\omega ~;~ \ddot{\omega} = - \frac{g \omega cos\theta}{L} - b \big( \frac{-g sin \theta}{L} - b \omega \big) \\
    &\theta_{b, L} = 0+0 = 0 ~;~ \omega_{b, L} = 0+\frac{g}{L^2}sin\theta
\end{align*}

\subsection{Lotka-Volterra System}\label{method:lv}
A simple biological dynamical system that describes the behavior of two species where one behaves as the predator and the other a prey is called the Lotka-Volterra system. It can be represented in two-dimensional form as follows:
\begin{align*}
    &\dot x = \alpha x - \beta xy ~;~ \dot y = -\gamma y + \delta xy ~;~~ (x, y) \in \Omega \\
    &\Omega \in [0, 1] \times [0, 1]
\end{align*}
where, $\alpha \in [0.6, 1]$ represent the per capita growth rate of the prey, $\beta \in [0.2, 0.5]$ represent the presence of predator in the prey death rate, $\gamma \in [0.5, 1.0]$ is the per capita death rate of the predator, and $\delta \in [0.1, 0.4]$ represents presence of prey in predator growth rate. Further, $x, y$ are the state variables for the system representing the prey and predator populations, respectively. The initial condition and parameters are uniform-randomly sampled from the range. The second derivative w.r.t time $t$ and the Lie derivatives w.r.t the parameters of the model $\alpha, \beta, \gamma, \delta$ are as follows:
\begin{align*}
    &\ddot{x} = -x(xy\delta-y\gamma)\beta + (x\alpha -xy \beta) (\alpha - y\beta) ~;~ \ddot{y} = (x\alpha - xy \beta) \gamma \delta + (xy\delta- y \gamma)( x\delta-\gamma) \\
    &x_{\alpha, \beta, \gamma, \delta} = x - xy ~;~~ y_{\alpha, \beta, \gamma, \delta} = -y + xy
\end{align*}

\subsection{Rikitake Attractor}
A dynamical system that models the behavior of a coupled magnetic dynamo, is called the Rikitake attractor. It can be simplified and represented in three-dimensional form as follows:  
\begin{align*}  
    &\dot{x} = -\mu x + yz ~;~ \dot{y} = -\mu y + x(z - h) ~;~ \dot{z} = 1 - xy ~;~~ (x, y, z) \in \Omega \\
    &\Omega \in [-0.5, 0.5] \times [-0.5, 0.5] \times [-0.5, 0.5]
\end{align*}
where, $\mu = (\omega_1 - \omega_2)\sqrt{CM/GL} \in [0.3, 0.9]$ consists of terms $C, G, L, M$ which are the moment of inertia, applied torque, self-inductance and mutual-inductance of the dynamos, which are rotated to $\omega_1, \omega_2$ angular velocities respectively. Further, $h = R\sqrt{C/GLM} \in [0.3, 0.9]$ is another simplified term, where $R$ is the electrical resistance. The initial condition and parameters are uniform-randomly sampled from the range. The second derivative with respect to time $t$ and the Lie derivatives with respect to the parameter $\mu$ of the model are omitted as they get larger to be represented. One can find the symbolic derivations of these systems in the code provided.

\subsection{Lorenz Attractor}
A dynamical system that models atmospheric convection and exhibits chaotic behavior is called the Lorenz attractor. It can be represented in three-dimensional form as follows:  
\begin{align*}  
    &\dot{x} = \sigma (y - x) ~;~ \dot{y} = x (\rho - z) - y ~;~ \dot{z} = xy - \beta z ~;~~ (x, y, z) \in \Omega \\
    &\Omega \in [0, 1] \times [0, 1] \times [0, 1]
\end{align*}  
where, $\sigma \in [0, 1]$ represent the Prandtl number, $\rho \in [0, 1]$ the Rayleigh number, and $\beta \in [0, 1]$ geometric factor. The variables $x, y, z$ correspond to the convective flow, temperature difference, and vertical temperature variation, respectively. The initial condition and parameters are uniform-randomly sampled from the range. The second derivative with respect to time $t$ and the Lie derivatives with respect to the parameters $\sigma, \rho, \beta$ of the model are as follows:  
\begin{align*}  
    &\ddot{x} = \sigma (\dot{y} - \dot{x}) ~;~ \ddot{y} = \dot{x} (\rho - z) + x (-\dot{z}) - \dot{y} ~;~ \ddot{z} = \dot{x} y + x \dot{y} - \beta \dot{z} \\  
    &x_{\sigma} = y - x ~;~~ y_{\rho} = x ~;~~ z_{\beta} = -z  
\end{align*}

\subsection{Susceptible-Infected-Recovered (SIR) Model}
The SIR model is an epidemiological framework used to describe the spread of diseases where individuals transition between being susceptible (S), infected (I), and recovered (R). It can be represented in three-dimensional form as follows:
\begin{align*}
    &\dot S = -\frac{IS \beta}{N} ~;~ \dot I = \frac{IS \beta}{N} - I \gamma ~;~ \dot R = I \gamma ~;~~~ (S, I, R) \in \Omega \\
    &\Omega \in [0, 1] \times [0, 1] \times [0, 1] 
\end{align*}
where, $\beta \in [0, 1]$ is the probability of disease transmission per contact, $\gamma \in [0, 1]$ is the per-capita recovery rate. The initial conditions and parameters of the model are uniform-randomly sampled from the range. The second derivatives w.r.t time $t$ and the Lie derivatives w.r.t. the parameters $\beta, \gamma$ are as follows:
\begin{align*}
    &\ddot S = \frac{K S \beta}{N} + L ~;~ \ddot I = - L + \frac{K(-N\gamma + S \beta)}{N} ~;~ \ddot R = K \gamma \\
    &S_{\beta, \gamma} = - \frac{SI}{N} ~;~ I_{\beta, \gamma} = \frac{SI}{N} - I ~;~ R_{\beta, \gamma} = I
\end{align*}
where, $K = \frac{IS \beta}{N} - I \gamma$, $L = \frac{\beta^2 I^2 S}{N^2}$

\subsection{Susceptible-Exposed-Infected-Recovered (SEIR) Model}
A compartmental epidemiological model that describes the spread of infectious diseases with temporary immunity is called the SEIR model. It can be represented in four-dimensional form as follows:  
\begin{align*}
    &\dot{S} = \mu (N - S) - \frac{\beta S I}{N} + \omega R ~;~~ \dot{E} = \frac{\beta S I}{N} - (\sigma + \mu) E ~; \\
    &\dot{I} = \sigma E - (\mu + \gamma + \alpha) I ~;~~ \dot{R} = \gamma I - (\mu + \omega) R ~;~~ S, E, I, R \in \Omega \\
    &\Omega \in [0, 0.99] \times [0, 0.99] \times [0, 0.5] \times [0, 0.5]
\end{align*}
where, $\mu \in [0.01, 0.02]$ represent the birth/death rate, $\beta \in [0.01, 0.02]$ transmission rate, $\sigma \in [0.1, 0.2]$ is the incubation rate, $\gamma \in [0.01, 0.2]$ is the recovery rate, $\alpha \in [0, 0.5]$ is the disease-induced death rate, and 
$\omega \in [0.1, 1]$ is the loss of immunity rate. Further, the variables $S, E, I, R$ correspond to the susceptible, exposed, infected, and recovered populations, respectively. The initial condition and parameters are uniform-randomly sampled from the range. The second derivative with respect to time $t$ and the Lie derivatives with respect to the parameter $\mu$ of the model are omitted as they get larger to be represented.

\section{Numerical Quadrature insights}\label{appendix:numericalquad}
To support our final approach of learning the remainder term, we proposed the numerical quadrature method. In this section, we define and compute the method against three systems~\ref{method:do},~\ref{method:dp},~\ref{method:lv} to showcase how approximation using quadratures affect the performance. We set the number of quadrature points to 10 and obtain the remainder term $R_m$ as follows:
\begin{align*}
    R_{m}(\vx_0, \vtheta, t) \approx \frac{4}{10} \left( \frac{1}{2} F(0) + F(0.1) + F(0.2) + \dots + \frac{1}{2} F(1)\right) \,
\end{align*}
The results shown in Tab.~\ref{tab:results-nq} indicate that for smaller number of quadrature points, the error increases rapidly over time making the predictions highly erroneous across all systems. However, since the number of points are small the algorithm takes less time to train the neural network as there is no longer expensive differential computation at each epoch. Further, Fig.~\ref{fig:nq-model-performance} shows the prediction performance of training the neural network using both the approaches on the first two dynamical system (we see similar performance for the Lotka-Volterra system as well).

\begin{table}[t]
\centering
\small
\begin{tabular}{lcclc}
\hline
Method & DO (2,1) & DP (2,2) & LV (2,4) & Time (sec) \\ \hline
TM-PINN & 0.003 & \textbf{0.012} & 0.004 & 1 \\
TM-PINN-NQ & \textbf{0.002} & 0.016 & \textbf{0.002} & 1 \\ \hline
TM-PINN & 0.048 & \textbf{0.185} & \textbf{0.06} & 2 \\
TM-PINN-NQ & \textbf{0.046} & 0.477 & 0.077 & 2 \\ \hline
TM-PINN & \textbf{0.170} & \textbf{0.677} & \textbf{0.453} & 3 \\
TM-PINN-NQ & 0.218 & 4.155 & 0.665 & 3 \\ \hline
\end{tabular}
\caption{Results showing (MAE$\downarrow$) on Taylor-Model PINN using 10 Numerical Quadrature points (TM-PINN-NQ) compared against TM-PINNs as reported in the main paper on three different dynamical systems. \{DO=Duffing Oscillator, DP=Damped Pendulum, LV=Lotka-Volterra\}.}
\label{tab:results-nq}
\end{table}

\begin{figure}[t]
    \centering
    \includegraphics[trim=0 15 10 10,clip,width=\linewidth]{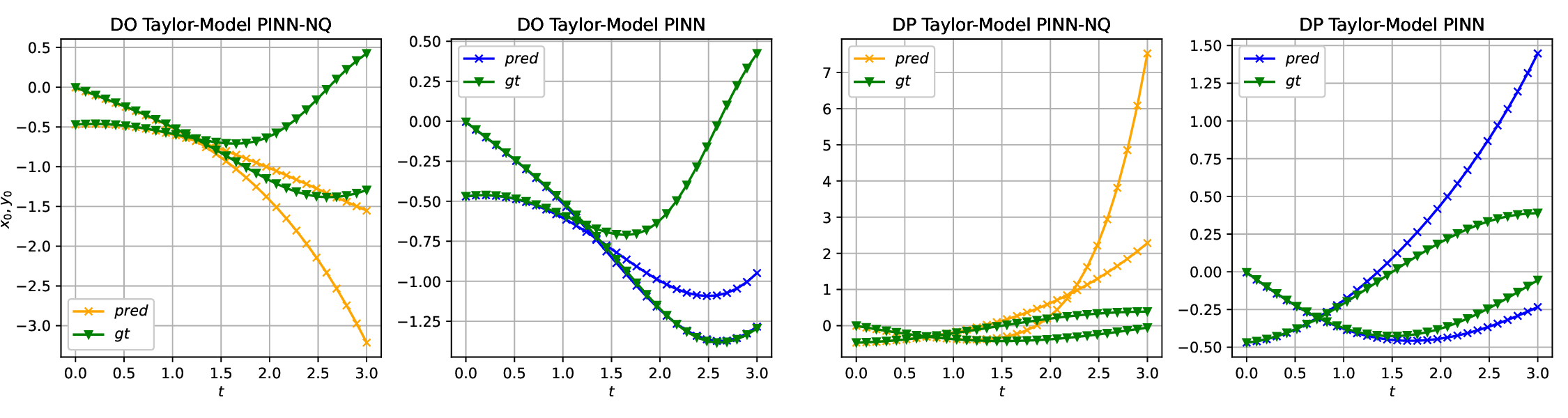}
    \caption{Prediction performance of the both models Taylor-Model PINN using 10 Numerical Quadrature points (TM-PINN-NQ) and TM-PINN, on \{DO=Duffing Oscillator\} and \{DP=Damped Pendulum\} systems. The initial conditions of the dynamical system are set as before.}
    \label{fig:nq-model-performance}
\end{figure}

\section{Supporting results from other systems} \label{appendix:A2}
In this section, we continue with the results from our experiments on the benchmark systems. First, we can continue looking at the prediction plots for the rest of the systems. Secondly, we look at the error propagation through time for the various models. Finally, we look at scaling the systems to larger dimensional ODEs and study the predictive performance in this scenario. 

\subsection{Prediction performance graphs}
The following Fig. \ref{fig:lv-model-performance}, \ref{fig:r-model-performance}, \ref{fig:seirs-model-performance}, \ref{fig:lo-model-performance} show the prediction performance of the three models across the same set of initial condition and parameter space. 
\begin{figure}[ht!]
    \centering
    \includegraphics[trim=0 15 10 10,clip,width=\linewidth]{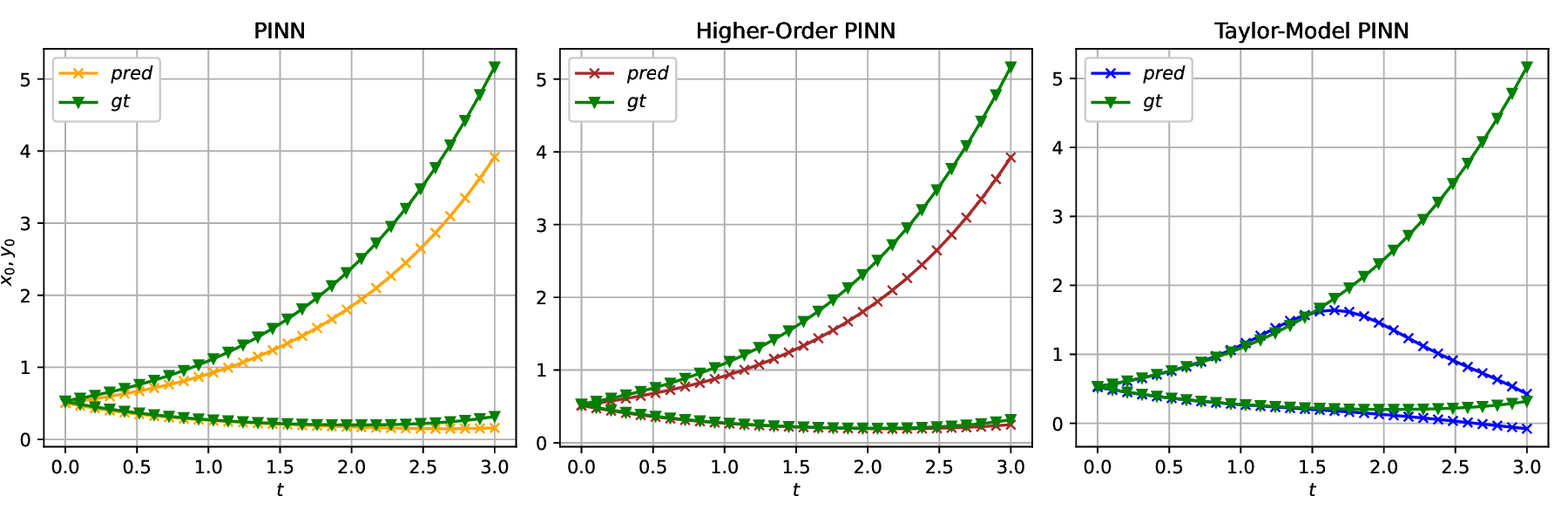}
    \caption{Prediction performance of the three models on \textbf{Lotka-Volterra} system. The initial conditions are set to $\vx_0=[0.53, 0.52]$ and $\vtheta_0=[0.87, 0.43, 0.95, 0.39]$}
    \label{fig:lv-model-performance}
\end{figure}

\begin{figure}[ht!]
    \centering
    \includegraphics[trim=0 15 10 10,clip,width=\linewidth]{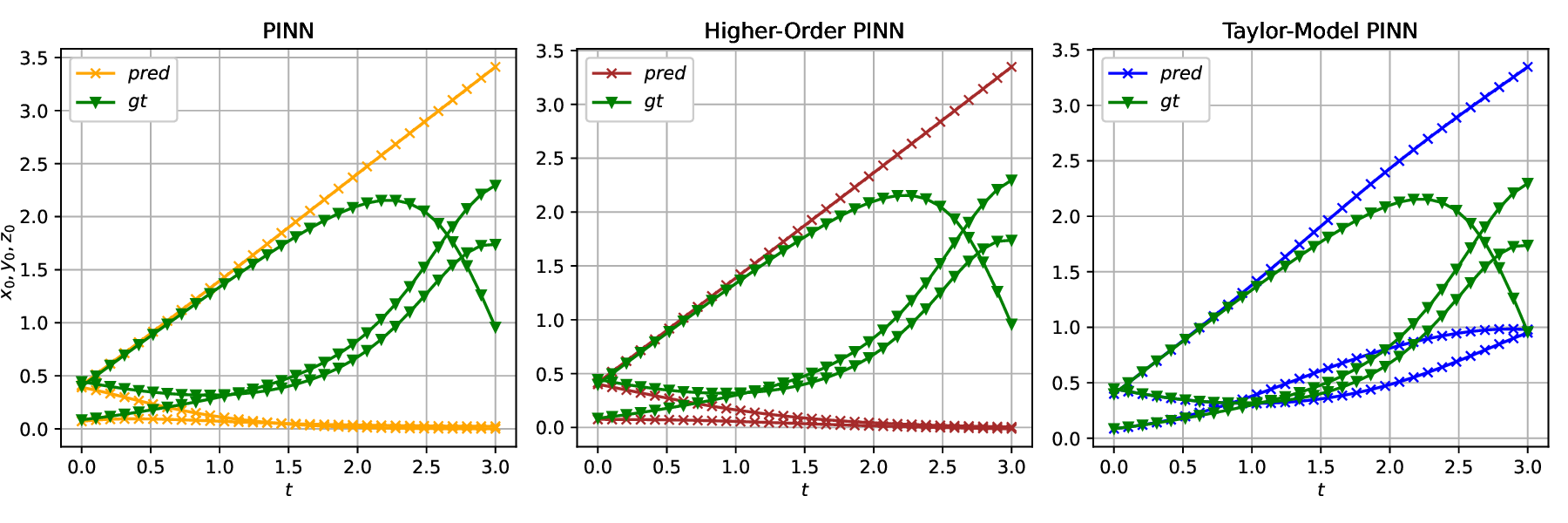}
    \caption{Prediction performance of the three models on \textbf{Rikitake} system. The initial conditions of the dynamical system are set to $\vx_0=[0.08, 0.44, 0.39]$ and $\vtheta_0=[0.49, 0.69]$}
    \label{fig:r-model-performance}
\end{figure}

\begin{figure}[ht!]
    \centering
    \includegraphics[trim=0 15 10 10,clip,width=\linewidth]{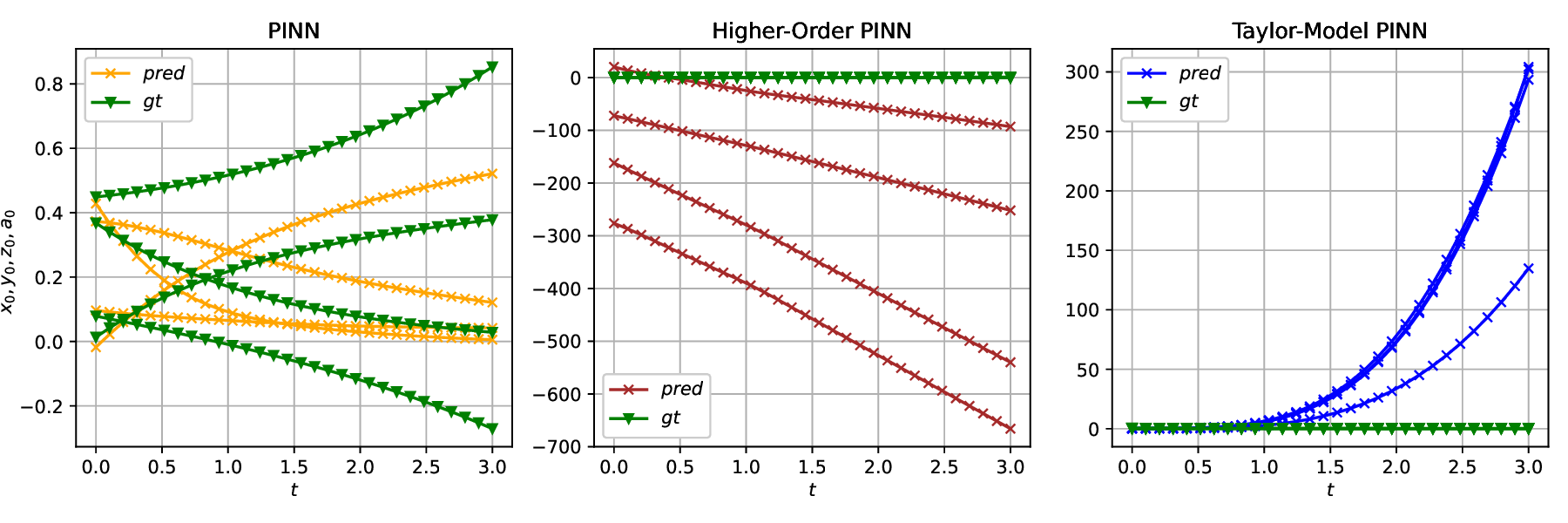}
    \caption{Prediction performance of the three models on \textbf{SEIRS} system. The initial conditions are set to $\vx_0=[0.01, 0.078, 0.45, 0.36]$ and $\vtheta_0=[0.01, 0.2, 0.17, 0.76, 0.36,
        0.049]$}
    \label{fig:seirs-model-performance}
\end{figure}

\begin{figure}[ht!]
    \centering
    \includegraphics[trim=0 15 10 10,clip,width=\linewidth]{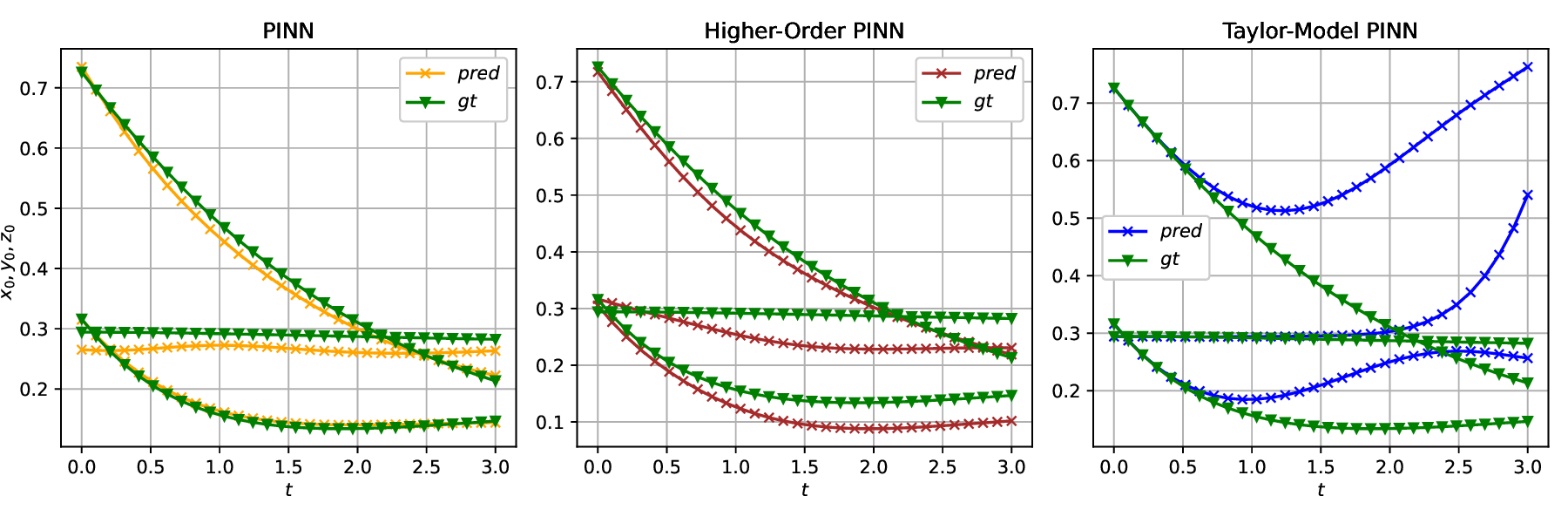}
    \caption{Prediction performance of the three models on \textbf{Lorenz Attractor} system. The initial conditions of the dynamical system are set to $\vx_0=[0.29, 0.32, 0.73]$ and $\vtheta_0=[0.53, 0.03, 0.79]$}
    \label{fig:lo-model-performance}
\end{figure}

\begin{figure}[ht!]
    \centering
    \includegraphics[trim=0 15 10 10,clip,width=\linewidth]{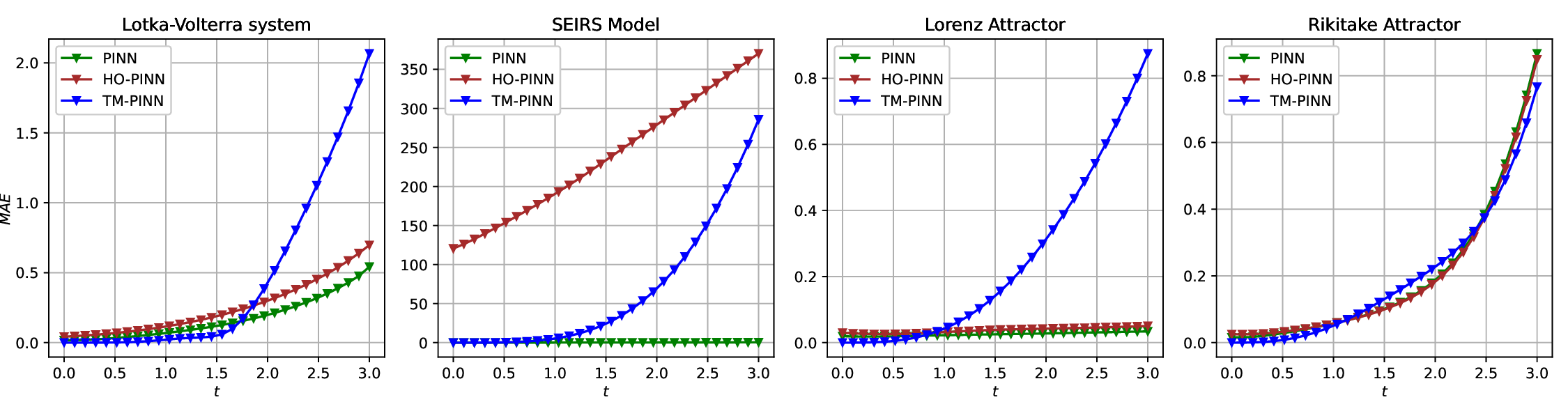}
    \caption{MAE plotted at various time points the three different models comparing the PINN, HO-PINN and the TM-PNN (our approach).}
    \label{fig:error-l1}
\end{figure}

\begin{table}[ht!]
\centering
\small
\begin{tabular}{lcccccccc}
\hline
Method & DO (2,1) & DP (2,2) & LV (2,4) & R (3,2) & SIR (3,2) & LoA (3,3) & SEIR (4,6) & Time (sec) \\ \hline
PINN & 0.178 & 0.154 & 0.051 & 0.055 & 0.120 & 0.027 & \textbf{0.102} & 1 \\
HO-PINN & 0.186 & 0.264 & 0.094 & 0.055 & 0.399 & 0.039 & 253.7 & 1 \\
TM-PINN & \textbf{0.008} & \textbf{0.046} & \textbf{0.007} & \textbf{0.028} & \textbf{0.004} & \textbf{0.021} & 1.346 & 1 \\ \hline
PINN & 0.307 & \textbf{0.195} & 0.136 & 0.149 & 0.202 & \textbf{0.033} & \textbf{0.198} & 2 \\
HO-PINN & 0.321 & 0.267 & 0.210 & 0.145 & 0.409 & 0.050 & 321.901 & 2 \\
TM-PINN & \textbf{0.119} & 0.664 & \textbf{0.123} & \textbf{0.142} & \textbf{0.045} & 0.167 & 20.78 & 2 \\ \hline
PINN & 0.411 & \textbf{0.214} & \textbf{0.349} & 0.469 & 0.255 & \textbf{0.040} & \textbf{0.299} & 3 \\
HO-PINN & 0.433 & 0.262 & 0.481 & 0.468 & 0.431 & 0.059 & 396.9 & 3 \\
TM-PINN & \textbf{0.337} & 2.130 & 0.912 & \textbf{0.371} & \textbf{0.150} & 0.478 & 92.44 & 3 \\ \hline
\end{tabular}
\caption{Results showing (RMSE$\downarrow$) on TM-PINN compared against vanilla PINN and HO-PINN on seven different dynamical system models across varying prediction time. \{DO=Duffing Oscillator, DP=Damped Pendulum, LV=Lotka-Volterra, R=Rikitake, SIR=Susceptible-Infected-Recovered, LoA=Lorenz Attractor, SEIR=S-Exposed-IR \}.}
\label{tab:results2}
\end{table}

\subsection{Results and error propagation}
In Fig.~\ref{fig:error-l1} we can further see that across all the remaining four systems, as noted earlier, TM-PINNs have negligent mean absolute error across the first few seconds of the evolution. But, the error tends to grow as prediction time horizon increases. Finally, we also compute the Root Mean Square Error (RMSE) and notice a similar performance across models Tab.~\ref{tab:results2}.

\subsection{Larger systems}\label{appendix:A2_3}
Addressing the reviews, we ran our method against two larger systems. (a) A multi-coupled damped oscillator, building upon the system provided in~\cite{ROS01a}, (b) Michaelis-Menten kinetics system similar to the system provided in~\cite{klipp2016systems}.

\subsubsection{Multi-Coupled Damped Oscillator}
The multi-coupled damped oscillator is a 8-dimensional nonlinear differential equation similar to the damped oscillator model used in our benchmarks (Appendix~\ref{appendix:A1}) and can be represented as follows:
\begin{align*}
    \dot{\mrx}_1 &= \mry_1 ~;~ \dot{\mry}_1 = \mu(1 - \mrx_1^2)\mry_1 - \mrx_1 + \delta(\mrx_2 - 2\mrx_1 + \mrx_4) \\
    \dot{\mrx}_2 &= \mry_2 ~;~ \dot{\mry}_2 = \mu(1 - \mrx_2^2)\mry_2 - \mrx_2 + \delta(\mrx_3 - 2\mrx_2 + \mrx_1) \\
    \dot{\mrx}_3 &= \mry_3 ~;~ \dot{\mry}_3 = \mu(1 - \mrx_3^2)\mry_3 - \mrx_3 + \delta(\mrx_4 - 2\mrx_3 + \mrx_2) \\
    \dot{\mrx}_4 &= \mry_4 ~;~ \dot{\mry}_4 = \mu(1 - \mrx_4^2)\mry_4 - \mrx_4 + \delta(\mrx_1 - 2\mrx_4 + \mrx_3) 
\end{align*}
where, $\{\mrx_1, \mry_1, \mrx_2, \mry_2, \mrx_3, \mry_3, \mrx_4, \mry_4\} \in [-0.5, 0.5]$ are the state variables of the system and $\delta, \mu \in [0.1, 0.5]$ are the damping factors of the coupled system. The initial conditions and parameters of the model are uniform-randomly sampled from the range. The second derivatives w.r.t. time $t$ and the Lie derivatives w.r.t the parameter of the model $\delta, \mu$ are symbolically created.

\subsubsection{Michaelis-Menton Enzyme Kinetics}
The Michaelis-Menten system explains how presence of some $i$-enzyme concentrations in an enzyme-substrate complex can cause kinetic rate enhancement of a reaction. We extend this system to 6-dimensions i.e. six enzymes present in the concentrate, which can be represented as follows:
\begin{align*}
    \dot{\mrx}_1 &= \frac{V_1}{K_m + 1} - \delta \mrx_1 ~;~ \dot{\mrx}_2 = \frac{V_2 \mrx_1}{K_m + \mrx_1} - \delta \mrx_2 \\
    \dot{\mrx_3} &= \frac{V_3 \mrx_2}{K_m + \mrx_2} - \delta \mrx_3 ~;~ \dot{\mrx}_4 = \frac{V_4 \mrx_3}{K_m + \mrx_3} - \delta \mrx_4 \\
    \dot{\mrx}_5 &= \frac{V_5 \mrx_4}{K_m+ \mrx_4} - \delta \mrx_5 ~;~ \dot{\mrx}_6 = \frac{V_6 \mrx_5}{K_m + \mrx_5} - \delta \mrx_6
\end{align*}
where, $\{\mrx_1, \mrx_2, \mrx_3, \mrx_4, \mrx_5, \mrx_6\} \in [0.1, 0.5]$ are the state variables and $\{V_1, V_2, V_3, V_4, V_5, V_6\} \in [0.5, 1.0]$ are the maximum reaction velocities for each enzyme present in the substrate. We set the Michaelis constant $K_m = 0.5$ and the degradation rate constant $\delta = 0.1$. The initial conditions and parameters of the model are uniform-randomly sampled from the range. The second derivatives w.r.t. time $t$ and the Lie derivatives w.r.t the parameter of the model $V_i$ are symbolically created.

To this end, we run all three models against the above two systems. The hyperparamters of the models are kept the same as other experiments reported, however, the learning rate is reduced to $0.005$ (high dimensional systems have sharper gradients and many local minima) and the time duration $T$ is reduced to 2 seconds with same intervals (to accomodate hardware limitations). We notice similar performance to previous methods where TM-PINNs perform well on shorter time horizons compared to other methods. Tab.~\ref{tab:results3} gives us the MAE across each system and Fig.~\ref{fig:cdo-model-performance},~\ref{fig:mmek-model-performance} show the prediction performance of the models.

\begin{table}[ht]
\centering
\small
\begin{tabular}{lccc}
\hline
Method & MMEK (6,6) & CDO (8,2) & Time (sec) \\ \hline
PINN & 0.020 & 0.061 & 1 \\
HO-PINN & 0.157 & 0.188 & 1 \\
TM-PINN & \textbf{0.005} & \textbf{0.005} & 1 \\ \hline
PINN & \textbf{0.030} & \textbf{0.112} & 2 \\
HO-PINN & 0.171 & 0.220 & 2 \\
TM-PINN & 0.097 & 0.130 & 2 \\ \hline
\end{tabular}
\caption{Results showing MAE($\downarrow$) on TM-PINN compared against vanilla PINN and HO-PINN on two larger dynamical system models across varying prediction time. \{MMEK=Michaelis-Menton Enzyme Kinematics, CDO=Coupled Damped Oscillators\}}
\label{tab:results3}
\end{table}

\begin{figure}[ht!]
    \centering
    \includegraphics[trim=0 15 10 10,clip,width=\linewidth]{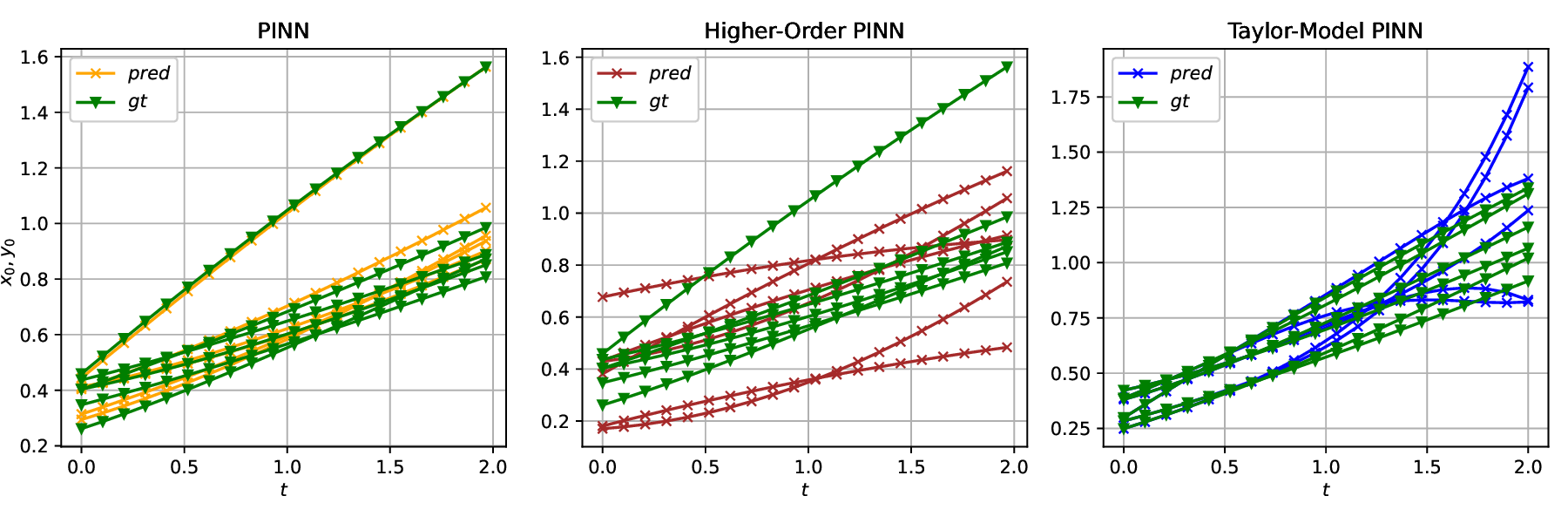}
    \caption{Prediction performance of the three models on \textbf{Michaelis-Menton Enzyme Kinematics} system. The initial conditions of the dynamical system are set to $\vx_0=[0.29, 0.38, 0.38, 0.28, 0.24, 0.42]$ and $\vtheta_0=[0.9, 0.76, 0.93, 0.62, 0.85, 0.74]$}
    \label{fig:mmek-model-performance}
\end{figure}

\begin{figure}[ht!]
    \centering
    \includegraphics[trim=0 15 10 10,clip,width=\linewidth]{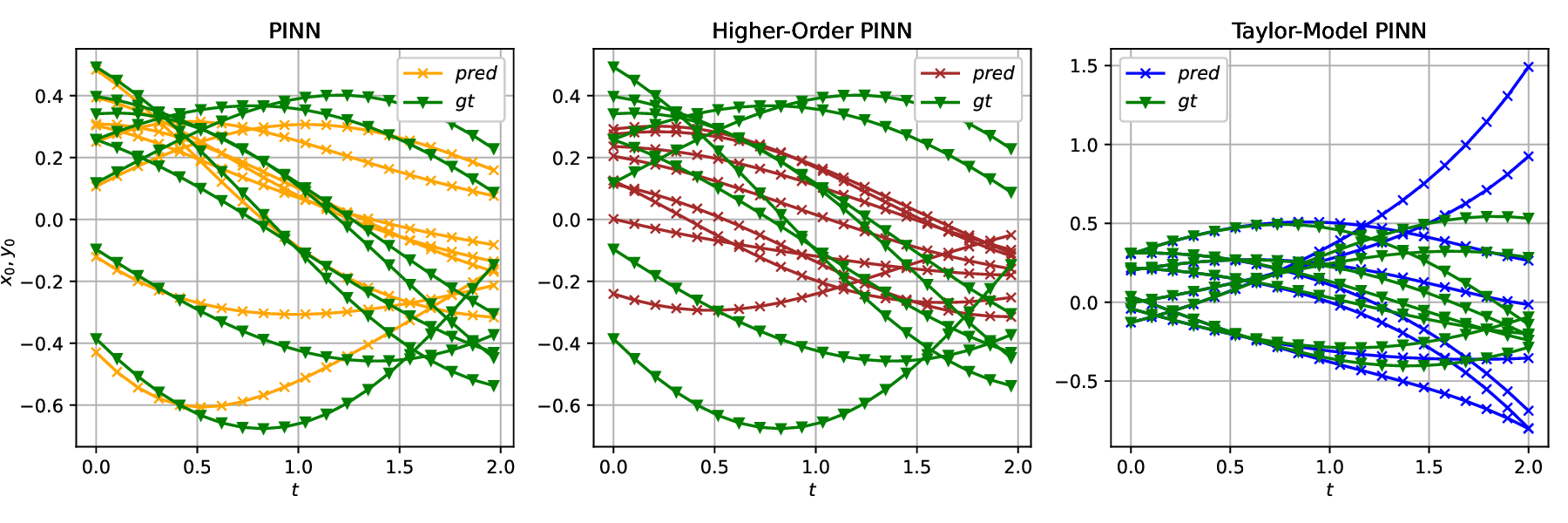}
    \caption{Prediction performance of the three models on \textbf{Coupled Damped Oscillator} system. The initial conditions of the dynamical system are set to $\vx_0=[0.00, 0.20, 0.21, -0.04, -0.13, 0.3, 0.31, 0.03]$ and $\vtheta_0=[0.44, 0.19]$}
    \label{fig:cdo-model-performance}
\end{figure}

\clearpage
\section{Algorithm for Taylor-Model PINNs} \label{appendix:A3}
In this section, we provide the algorithm for training TM-PINNs as used in our experiments. We start by symbolically computing the $m$ Lie derivatives of the dynamical system $f$ provided, and use a neural network $\psi$ to numerically approximate the remainder term. The loss functions, defined in \ref{par:tmpinnloss} is implemented as shown in Alg.~\ref{alg:tm-pinn} pseudocode.

\begin{algorithm}[ht!]
\caption{Taylor Model Physics-Informed Neural Networks (TM-PINNs)}
\label{alg:tm-pinn}
\begin{algorithmic}[1]
\STATE \textbf{Require:} Training data $\mathcal{D} = (\vx_i^0, \vtheta_i^0, \vt)_{i=1}^N$ where $\vx_0$ and $\vtheta_0$ are the initial conditions and parameters, $\vt$ represents the time horizon up to time $T$ divided by $\Delta t$, number of epochs $N_{\text{iter}}$, the dynamical system $f(\vx_0, \vtheta_0, t)$, and a neural network $\psi$ with weights and biases $(\vec{w}, \vec{b})$.
\STATE \textbf{Initialize:} Symbolically compute the $m$ Lie derivatives $\Lie^{(m)}$ of the system $f$
\STATE \textbf{Training:}
\FOR{$i=1$ to $N_{\text{iter}}$}
\STATE Sample a random mini-batch of training data $\vb \subset \mathcal{D}$
\STATE $\vec{g}_r \leftarrow \vb + t \Lie^{(1)}(\vb) + \dots + \frac{t^m}{m!} \Lie^{(m)}(\vb)$
\STATE $\vec{g}_l \leftarrow \Lie^{(1)}(\vb) + t \Lie^{(2)}(\vb) + \dots + \frac{t^{(m-1)}}{(m-1)!} \Lie^{(m)}(\vb)$
\STATE $\hat{\vb}, \dot{\vb} \leftarrow \psi(\vb), \nabla \psi(\vb)$
\STATE $\vec{g}_r \leftarrow f(\vec{g}_r + \frac{t^{(m+1)}}{(m+1)!} \hat{\vb}, \vtheta_0, t)$
\STATE $\vec{g}_l \leftarrow \vec{g}_l +\frac{t^m}{m!} \hat{\vb} + \frac{t^{(m+1)}}{(m+1)!} \dot{\vb}$
\STATE $\nabla_{\vec{w}, \vec{b}}\mathcal{L}(\vec{w}, \vec{b}) \leftarrow \nabla_{\vec{w}, \vec{b}} (|| \vec{g}_l - \vec{g}_r ||^2_2) + \nabla_{\vec{w}, \vec{b}} (|| \Lie^{(m+1)}(\vb) - \hat{\vb} ||^2_2)$ 
\STATE Update weights of $\psi$ using ADAM
\ENDFOR
\end{algorithmic}
\end{algorithm}

Here, we assume the same gradient update steps as used by the Adaptive Moment Estimation (ADAM) optimizer~\cite{kingma2014adam}, with some learning rate $\eta$. All the neural network weights are initalized using the Glorot initialization~\cite{glorot2010understanding} i.e. sampled from a uniform random distribution of mean $0$ and variance $\sqrt{\frac{2}{\xi_{in} + \xi_{out}}}$, where $\xi_{in}$ in the number of input layers and $\xi_{out}$ is the number of output layers.

\end{appendices}

%% file: main.bbl
\begin{thebibliography}{32}
\providecommand{\natexlab}[1]{#1}
\providecommand{\url}[1]{\texttt{#1}}
\expandafter\ifx\csname urlstyle\endcsname\relax
  \providecommand{\doi}[1]{doi: #1}\else
  \providecommand{\doi}{doi: \begingroup \urlstyle{rm}\Url}\fi

\bibitem[Agarwal and O'Regan(2009)]{Agarwal+ORegan/2009/Ordinary}
Ravi~P. Agarwal and Donal O'Regan.
\newblock \emph{Ordinary and Partial Differential Equations: With Special Functions, Fourier Series, and Boundary Value Problems}.
\newblock Universitext. Springer New York, NY, 1 edition, 2009.

\bibitem[Althoff(2015)]{Althoff/2015/CORA}
M.~Althoff.
\newblock An introduction to cora 2015.
\newblock In \emph{Proc. of ARCH'15}, volume~34 of \emph{EPiC Series in Computer Science}, pages 120--151. EasyChair, 2015.

\bibitem[Althoff et~al.(2021)Althoff, Frehse, and Girard]{Althoff+Frehse+Girard/2021/Set}
Matthias Althoff, Goran Frehse, and Antoine Girard.
\newblock Set propagation techniques for reachability analysis.
\newblock \emph{Annual Review of Control, Robotics, and Autonomous Systems}, 4, 2021.

\bibitem[Angeloudi et~al.(2024)Angeloudi, Audenaert, Bowles, Boyd, Chemaly, Cherinka, Ciucă, Cranmer, Do, Grayling, Hayes, Hehir, Ho, Huertas-Company, Iyer, Jablonska, Lanusse, Leung, Mandel, Martínez-Galarza, Melchior, Meyer, Parker, Qu, Shen, Smith, Stone, Walmsley, and Wu]{angeloudi_multimodal_nodate}
Eirini Angeloudi, Jeroen Audenaert, Micah Bowles, Benjamin~M Boyd, David Chemaly, Brian Cherinka, Ioana Ciucă, Miles Cranmer, Aaron Do, Matthew Grayling, Erin~E Hayes, Tom Hehir, Shirley Ho, Marc Huertas-Company, Kartheik~G Iyer, Maja Jablonska, Francois Lanusse, Henry~W Leung, Kaisey Mandel, Juan~Rafael Martínez-Galarza, Peter Melchior, Lucas Meyer, Liam~H Parker, Helen Qu, Jeff Shen, Michael~J Smith, Connor Stone, Mike Walmsley, and John~F Wu.
\newblock The {Multimodal} {Universe}: {Enabling} {Large}-{Scale} {Machine} {Learning} with 100 {TB} of {Astronomical} {Scientific} {Data}.
\newblock 2024.

\bibitem[Apostol(1991)]{Apostol/1967/Calculus}
Tom~M. Apostol.
\newblock \emph{Calculus, Vol. 1: One-Variable Calculus, with an Introduction to Linear Algebra}.
\newblock John Wiley \& Sons, New York, 2nd edition, 1991.

\bibitem[Balduzzi et~al.(2016)Balduzzi, McWilliams, and Butler-Yeoman]{balduzzi_mcwilliams_butler-yeoman_2016}
David Balduzzi, Brian McWilliams, and Tony Butler-Yeoman.
\newblock Neural taylor approximations: Convergence and exploration in rectifier networks, 2016.

\bibitem[Basor and Morrison(2024)]{Basor+Morrison/697/Analytic}
Estelle Basor and Rebecca Morrison.
\newblock Analytic solutions to nonlinear odes via spectral power series.
\newblock \emph{Linear Algebra and its Applications}, 697:\penalty0 561--582, Sep 2024.

\bibitem[Bellman(1943)]{Bellman/1943/Stability}
Richard Bellman.
\newblock The stability of solutions of linear differential equations.
\newblock \emph{Duke Math. J.}, 10\penalty0 (4):\penalty0 643--647, 1943.

\bibitem[Berz and Makino(1998)]{Berz+Makino/1998/Verified}
M.~Berz and K.~Makino.
\newblock Verified integration of {ODE}s and flows using differential algebraic methods on high-order {T}aylor models.
\newblock \emph{Reliable Computing}, 4:\penalty0 361--369, 1998.

\bibitem[Chen and Sankaranarayanan(2022)]{Chen+Sankaranarayanan/2022/Reachability}
Xin Chen and Sriram Sankaranarayanan.
\newblock Reachability analysis for cyber-physical systems: Are we there yet? (invited paper).
\newblock In \emph{Proc. NASA Formal Methods Symposium}, volume 13260 of \emph{Lecture Notes in Computer Science}, page 109–130. Springer, 2022.

\bibitem[Glorot and Bengio(2010)]{glorot2010understanding}
Xavier Glorot and Yoshua Bengio.
\newblock Understanding the difficulty of training deep feedforward neural networks.
\newblock In \emph{Proceedings of the thirteenth international conference on artificial intelligence and statistics}, pages 249--256. JMLR W\&CP, 2010.

\bibitem[Hairer et~al.(1993)Hairer, Wanner, and N{\o}rsett]{Hairer+Others/2000/Solving}
E.~Hairer, G.~Wanner, and S.~P. N{\o}rsett.
\newblock \emph{Solving Ordinary Differential Equations {I}: Nonstiff Problems}.
\newblock Springer, Berlin, second edition, 1993.

\bibitem[Hwang and Lim(2025)]{hwang_dual_2025}
Youngsik Hwang and Dong-Young Lim.
\newblock Dual {Cone} {Gradient} {Descent} for {Training} {Physics}-{Informed} {Neural} {Networks}, jan 2025.
\newblock arXiv:2409.18426 [cs].

\bibitem[Iglesias-Suarez et~al.(2024)Iglesias-Suarez, Gentine, Solino-Fernandez, Beucler, Pritchard, Runge, and Eyring]{iglesias-suarez_causally-informed_2024}
Fernando Iglesias-Suarez, Pierre Gentine, Breixo Solino-Fernandez, Tom Beucler, Michael Pritchard, Jakob Runge, and Veronika Eyring.
\newblock Causally-{Informed} {Deep} {Learning} to {Improve} {Climate} {Models} and {Projections}.
\newblock \emph{Journal of Geophysical Research: Atmospheres}, 129\penalty0 (4):\penalty0 e2023JD039202, 2024.
\newblock ISSN 2169-8996.
\newblock \doi{10.1029/2023JD039202}.
\newblock \_eprint: https://onlinelibrary.wiley.com/doi/pdf/10.1029/2023JD039202.

\bibitem[Kingma and Ba(2014)]{kingma2014adam}
Diederik~P. Kingma and Jimmy Ba.
\newblock Adam: A method for stochastic optimization, 2014.

\bibitem[Klipp et~al.(2016)Klipp, Liebermeister, Wierling, Kowald, and Herwig]{klipp2016systems}
Edda Klipp, Wolfram Liebermeister, Christoph Wierling, Axel Kowald, and Ralf Herwig.
\newblock \emph{Systems Biology: A Textbook}.
\newblock Wiley-VCH, 2nd edition, 2016.
\newblock ISBN 9783527336364.
\newblock See Chapter 5: Modeling Biochemical Reactions — examples of cascades with Michaelis--Menten steps.

\bibitem[Kong et~al.(2015)Kong, Gao, Chen, and Clarke]{Kong+/2015/dReach}
S.~Kong, S.~Gao, W.~Chen, and E.~M. Clarke.
\newblock dreach: {\(\delta\)}-reachability analysis for hybrid systems.
\newblock In \emph{Proc. of TACAS'15}, volume 9035 of \emph{LNCS}, pages 200--205. Springer, 2015.

\bibitem[Krishnapriyan et~al.(2021)Krishnapriyan, Gholami, Zhe, Kirby, and Mahoney]{krishnapriyan_characterizing_2021}
Aditi Krishnapriyan, Amir Gholami, Shandian Zhe, Robert Kirby, and Michael~W Mahoney.
\newblock Characterizing possible failure modes in physics-informed neural networks.
\newblock In \emph{Advances in {Neural} {Information} {Processing} {Systems}}, volume~34, pages 26548--26560. Curran Associates, Inc., 2021.

\bibitem[Li et~al.(2023)Li, Jiang, Zhang, and Zhu]{li_medical_2023}
Mengfang Li, Yuanyuan Jiang, Yanzhou Zhang, and Haisheng Zhu.
\newblock Medical image analysis using deep learning algorithms.
\newblock \emph{Front Public Health}, 11:\penalty0 1273253, nov 2023.
\newblock ISSN 2296-2565.
\newblock \doi{10.3389/fpubh.2023.1273253}.

\bibitem[Makino and Berz(2009)]{Makino+Berz/2009/Taylor}
K.~Makino and M.~Berz.
\newblock Rigorous integration of flows and {ODE}s using {T}aylor models.
\newblock In \emph{Proc.\ SNC'09}, pages 79--84, 2009.

\bibitem[Raissi et~al.(2019)Raissi, Perdikaris, and Karniadakis]{raissi_physics-informed_2019}
M.~Raissi, P.~Perdikaris, and G.~E. Karniadakis.
\newblock Physics-informed neural networks: {A} deep learning framework for solving forward and inverse problems involving nonlinear partial differential equations.
\newblock \emph{Journal of Computational Physics}, 378:\penalty0 686--707, feb 2019.
\newblock ISSN 0021-9991.
\newblock \doi{10.1016/j.jcp.2018.10.045}.

\bibitem[Raissi et~al.(2018)Raissi, Perdikaris, and Karniadakis]{raissi2018multistepneuralnetworksdatadriven}
Maziar Raissi, Paris Perdikaris, and George~Em Karniadakis.
\newblock Multistep neural networks for data-driven discovery of nonlinear dynamical systems, 2018.

\bibitem[Rosenblum et~al.(2001)Rosenblum, Pikovsky, and Kurths]{ROS01a}
M.~G. Rosenblum, A.~Pikovsky, and J.~Kurths.
\newblock \emph{Synchronization -- A universal concept in nonlinear sciences}.
\newblock Cambridge University Press, Cambridge, 2001.

\bibitem[Shukla et~al.(2020)Shukla, Di~Leoni, Blackshire, Sparkman, and Karniadakis]{shukla_physics-informed_2020}
Khemraj Shukla, Patricio~Clark Di~Leoni, James Blackshire, Daniel Sparkman, and George~Em Karniadakis.
\newblock Physics-{Informed} {Neural} {Network} for {Ultrasound} {Nondestructive} {Quantification} of {Surface} {Breaking} {Cracks}.
\newblock \emph{J Nondestruct Eval}, 39\penalty0 (3):\penalty0 61, aug 2020.
\newblock ISSN 1573-4862.
\newblock \doi{10.1007/s10921-020-00705-1}.

\bibitem[Son et~al.(2023)Son, Cho, and Hwang]{son_enhanced_2023}
Hwijae Son, Sung~Woong Cho, and Hyung~Ju Hwang.
\newblock Enhanced physics-informed neural networks with {Augmented} {Lagrangian} relaxation method ({AL}-{PINNs}).
\newblock \emph{Neurocomputing}, 548:\penalty0 126424, sep 2023.
\newblock ISSN 0925-2312.
\newblock \doi{10.1016/j.neucom.2023.126424}.

\bibitem[Steger et~al.(2022)Steger, Rohrhofer, and Geiger]{steger2022how}
Sophie Steger, Franz~M. Rohrhofer, and Bernhard~C Geiger.
\newblock How {PINN}s cheat: Predicting chaotic motion of a double pendulum.
\newblock In \emph{The Symbiosis of Deep Learning and Differential Equations II}, 2022.

\bibitem[Wang et~al.(2022)Wang, Li, He, and Wang]{wang_is_nodate}
Chuwei Wang, Shanda Li, Di~He, and Liwei Wang.
\newblock Is {L2} {Physics}-{Informed} {Loss} {Always} {Suitable} for {Training} {Physics}-{Informed} {Neural} {Network}?
\newblock 2022.

\bibitem[Wang and Perdikaris(2021)]{wang_deep_2021}
Sifan Wang and Paris Perdikaris.
\newblock Deep learning of free boundary and {Stefan} problems.
\newblock \emph{Journal of Computational Physics}, 428:\penalty0 109914, mar 2021.
\newblock ISSN 00219991.
\newblock \doi{10.1016/j.jcp.2020.109914}.
\newblock arXiv:2006.05311 [math].

\bibitem[Xiang et~al.(2025)Xiang, Peng, Yao, Liu, and Zhang]{xiang_physics-informed_2025}
Zixue Xiang, Wei Peng, Wen Yao, Xu~Liu, and Xiaoya Zhang.
\newblock Physics-informed {Neural} {Implicit} {Flow} neural network for parametric {PDEs}.
\newblock \emph{Neural Netw}, 185:\penalty0 107166, jan 2025.
\newblock ISSN 1879-2782.
\newblock \doi{10.1016/j.neunet.2025.107166}.

\bibitem[Yin et~al.(2021)Yin, Zheng, Humphrey, and Karniadakis]{yin_non-invasive_2021}
Minglang Yin, Xiaoning Zheng, Jay~D. Humphrey, and George~Em Karniadakis.
\newblock Non-invasive {Inference} of {Thrombus} {Material} {Properties} with {Physics}-informed {Neural} {Networks}.
\newblock \emph{Computer Methods in Applied Mechanics and Engineering}, 375:\penalty0 113603, mar 2021.
\newblock ISSN 00457825.
\newblock \doi{10.1016/j.cma.2020.113603}.
\newblock arXiv:2005.11380 [physics].

\bibitem[Zhu et~al.(2022)Zhu, Jing, Leve, and Ferrari]{zhu_jing_leve_ferrari_2022}
Frances Zhu, Dongheng Jing, Frederick Leve, and Silvia Ferrari.
\newblock Nn-poly: Approximating common neural networks with taylor polynomials to imbue dynamical system constraints.
\newblock \emph{Frontiers in Robotics and AI}, 9, Nov 2022.
\newblock \doi{https://doi.org/10.3389/frobt.2022.968305}.

\bibitem[Zwerschke et~al.(2024)Zwerschke, Weyrauch, Götz, and Debus]{zwerschke_weyrauch_götz_debus_2024}
Pavel Zwerschke, Arvid Weyrauch, Markus Götz, and Charlotte Debus.
\newblock Taylor expansion in neural networks: How higher orders yield better predictions.
\newblock \emph{Iospress.nl}, page 2983–2989, 2024.
\newblock \doi{https://doi.org/10.3233/FAIA240838}.

\end{thebibliography}
